\pdfoutput=1
\documentclass[twoside]{article}

\usepackage[accepted]{aistats2020}

\usepackage[round]{natbib}

\usepackage[utf8]{inputenc}
\usepackage{url,graphicx,subfigure,color,appendix,enumitem}
\usepackage{algorithm}
\usepackage{amsmath,amssymb,amsthm,amsfonts,latexsym,mathtools,bm,datetime}
\usepackage{microtype}
\usepackage[table,xcdraw]{xcolor}
\usepackage{mathrsfs}
\usepackage{algpseudocode}
\usepackage{xspace}

\definecolor{mydarkblue}{rgb}{0,0.08,0.45}
\usepackage[colorlinks=true,
    linkcolor=mydarkblue,
    citecolor=mydarkblue,
    filecolor=mydarkblue,
    urlcolor=mydarkblue]{hyperref} 
\usepackage[capitalise]{cleveref}

\newcommand{\x}{w}
\newcommand{\xk}{w_{k}}

\newcommand{\xkk}{w_{k+1}}
\newcommand{\xopt}{w^{*}}

\newcommand{\grad}[1]{\nabla f(#1)}

\newcommand{\norm}[1]{\left\|#1\right\|}
\newcommand{\normsq}[1]{\left\|#1\right\|^{2}}

\newcommand{\E}{\mathbb{E}}

\newcommand{\etak}{\eta_{k}}

\newcommand{\mutilde}{\tilde{\mu}}
\newcommand{\Ltilde}{\tilde{L}}
\newcommand{\mubar}{\bar{\mu}}

\newcommand{\indnorm}[2]{\left\|#1\right\|_{#2}}

\newcommand{\Hk}{\mathbf{B}_k}

\newcommand{\transpose}{^\mathsf{\scriptscriptstyle T}}

\newtheorem{theorem}{Theorem}
\newtheorem{lemma}{Lemma}
\newtheorem{proposition}{Proposition}
\newtheorem{corollary}{Corollary}

\newcommand{\R}{\mathbb{R}}

\newcommand{\inner}[2]{\left\langle#1\,,\,#2\right\rangle}

\newcommand{\dom}{\text{dom}}

\newcommand{\defeq}{\vcentcolon=}

\newcommand*{\ie}{i.e.\@\xspace}

\newcommand{\gradf}[1]{\nabla f(#1)}
\newcommand{\sgradf}[2]{\nabla f_{#1}(#2)}
\newcommand{\hessf}[1]{\nabla^2 f(#1)}

\newcommand{\shessf}[2]{\mathbf{H}_{#1}(#2)}
\newcommand{\tensf}[1]{f^{\prime \prime \prime}(#1)}

\newcommand{\batchG}{\mathcal{G}}
\newcommand{\batchS}{\mathcal{S}}
\newcommand{\Gk}{\mathcal{G}_k}
\newcommand{\Sk}{\mathcal{S}_k}

\newcommand{\expectationof}[1]{\mathbb{E}\left[#1\right]}
\newcommand{\expectationwrtof}[2]{\mathbb{E}_{#1}\left[#2\right]}
\makeatletter
\newcommand*{\rom}[1]{\expandafter\@slowromancap\romannumeral #1@}
\makeatother
\makeatletter
\newcommand{\rnum}[1]{\lowercase\expandafter{\romannumeral #1\relax}}
\makeatother

\newcommand{\rhond}{\rho_{\text{nd}}}
\newcommand{\nd}[1]{\lambda(#1)}

\newcommand{\omegaof}[1]{\omega\left(#1\right)}
\newcommand{\omegastarof}[1]{\omega_*\left(#1\right)}
\newcommand{\omegaprimeof}[1]{\omega'\left(#1\right)}

\usepackage[compact]{titlesec}

\begin{document}
\runningtitle{Stochastic Second-Order Methods under Interpolation}
\runningauthor{Meng, Vaswani, Laradji, Schmidt, Lacoste-Julien}
\twocolumn[
\aistatstitle{Fast and Furious Convergence: \\ Stochastic Second-Order Methods under Interpolation}

\aistatsauthor{Si Yi Meng$^{*1}$ \And Sharan Vaswani$^{*2}$  \And Issam Laradji$^{1,3}$  \And Mark Schmidt$^{1,4}$ \And  Simon Lacoste-Julien$^{2}$ }

\aistatsaddress{ ${^1}$University of British Columbia\\ ${^2}$Mila, Universit\'e de Montr\'eal  \And ${^3}$Element AI \\ ${^4}$CCAI Affiliate Chair (Amii) } 
]
\begin{abstract}
We consider stochastic second-order methods for minimizing smooth and strongly-convex functions under an interpolation condition satisfied by over-parameterized models. Under this condition, we show that the regularized subsampled Newton method (R-SSN) achieves global linear convergence with an adaptive step-size and a constant batch-size. By growing the batch size for both the subsampled gradient and Hessian, we show that R-SSN can converge at a quadratic rate in a local neighbourhood of the solution. We also show that R-SSN attains local linear convergence for the family of self-concordant functions. Furthermore, we analyze stochastic BFGS algorithms in the interpolation setting and prove their global linear convergence. We empirically evaluate stochastic L-BFGS and a ``Hessian-free'' implementation of R-SSN for binary classification on synthetic, linearly-separable datasets and real datasets under a kernel mapping. Our experimental results demonstrate the fast convergence of these methods, both in terms of the number of iterations and wall-clock time.
\end{abstract}

\section{Introduction}
\label{sec:introduction}
Many machine learning tasks can be formulated as minimizing a finite sum of objective functions. For instance, in supervised learning, each component of the finite sum is the loss function associated with a training example. The overall objective is typically minimized using stochastic first-order methods such as stochastic gradient descent (SGD) or its variance-reduced versions~\citep{schmidt2017minimizing,johnson2013accelerating,defazio2014saga}. However, first-order methods can suffer from slow convergence on ill-conditioned problems. Popular adaptive methods~\citep{duchi2011adaptive, kingma2014adam, tieleman2012lecture} alleviate this problem to some extent by using the covariance of stochastic gradients to approximate second order information, although there are no worst-case guarantees on the quality of this approximation. Empirically, these adaptive methods are more robust to the problem's conditioning and result in decent performance across different tasks. 

There has been substantial work in explicitly incorporating second-order information through the Hessian matrix in Newton methods~\citep{liu1989limited, nocedal2006numerical, martens2010deep} or through the Fisher information matrix in natural gradient methods~\citep{amari1998natural, roux2008topmoumoute,pascanu2013revisiting, martens2015optimizing}. Incorporating this information to adapt to the local curvature enables Newton methods to achieve quadratic convergence in a neighbourhood of the solution, whereas typical first-order methods can only achieve linear convergence~\citep{nesterov2018lectures}. However, the computational complexity of the Newton method is linear in the number of training examples and cubic in the problem dimension, making it impractical for large high-dimensional datasets. 

In order to reduce the dependence on the number of training examples, \emph{subsampled Newton} methods~\citep{roosta2016suba, roosta2016subb, bollapragada2018exact, erdogdu2015convergence, xu2016sub} take a small batch of examples to form the subsampled Hessian in each iteration. In order to further scale these methods to high-dimensional datasets or models with a large number of parameters, there has been a rich body of research in developing computationally efficient approximations to the Hessian matrix~\citep{becker1988improving, liu1989limited, martens2015optimizing, erdogdu2015convergence, pilanci2017newton, patel2016kalman}. In this paper, we mainly restrict our attention to subsampled Newton methods and consider problems of a moderate dimension where it is computationally efficient to perform subsampled Hessian-vector operations~\citep{werbos1988backpropagation, pearlmutter1994fast}. 

Recently, a parallel line of work has focused on optimizing large over-parameterized models prevalent in modern machine learning. The high expressivity of these models enables them to \emph{interpolate} the training data, meaning that the gradient for each component function in the finite sum can become zero at the solution of the overall objective~\citep{schmidt2013fast,vaswani2018fast, vaswani2019painless, ma2018power, bassily2018exponential}. Examples of such models include logistic regression on linearly-separable data, non-parametric regression~\citep{liang2018just, belkin2018does}, boosting ~\citep{schapire1998boosting}, and over-parameterized deep neural networks~\citep{zhang2016understanding}. Interpolation allows stochastic first-order methods to attain the convergence rates of their deterministic counterparts, without the need for explicit variance reduction~\citep{vaswani2018fast, vaswani2019painless, ma2018power, bassily2018exponential}. However, first order methods do not take advantage of second order information and therefore, can still be hindered by ill-conditioning. \citet{bertsekas1997nonlinear} shows that the Gauss-Newton method is equivalent to Newton's method for solving nonlinear least-squares when interpolation is satisfied; hence the convergence is also superlinear. These reasons motivate us to study stochastic second-order methods in the interpolation setting.

\subsection{Contributions}
\label{sec:contributions}
We focus on the regularized subsampled Newton method (R-SSN) that uses a batch of examples to form a subsampled Hessian and gradient vector in each iteration. R-SSN combines the subsampled Hessian with the Levenberg-Marquart (LM) regularization~\citep{levenberg1944method,marquardt1963algorithm} that ensures a unique update direction. We first analyze the convergence rate of R-SSN for strongly-convex functions in the interpolation setting. In~\cref{sec:ssn-global}, we show that R-SSN with an adaptive step-size and a \emph{constant} batch-size can achieve global linear convergence in expectation. This is in contrast with the work of \citet{bollapragada2018exact} and \citet{bellavia2018subsampled} that analyze subsampled Newton methods without interpolation, and require a \emph{geometrically increasing} batch-size to achieve global linear convergence.

If we allow for a growing batch-size in the interpolation setting, R-SSN results in linear-quadratic convergence in a local neighbourhood of the optimal solution (\cref{sec:ssn-local}). In contrast, in order to obtain superlinear convergence, \citet{bollapragada2018exact} require the batch size for the subsampled gradient to grow at a faster-than-geometric rate. Our results thus show that interpolation allows R-SSN to achieve fast convergence with a more reasonable growth of the batch size. 

We further analyze the performance of R-SSN for minimizing self-concordant functions. In~\cref{sec:theory-self-concordance}, we analyze the convergence of R-SSN for self-concordant functions under the interpolation setting, and show that R-SSN with an adaptive step-size and constant batch-size results in local linear convergence in expectation. Closest to our work is the recent paper by \citet{marteau2019globally} that shows approximate Newton methods result in local linear convergence independent of the condition number; however, they do not consider the interpolation setting and requires that the approximate Newton directions are ``close'' to the true direction with high probability. 

Next, we study stochastic Quasi-Newton methods~\citep{schraudolph2007stochastic, sunehag2009variable,zhou2017stochastic, liu2018acceleration, berahas2019quasi, gao2019quasi} in the interpolation setting (\cref{sec:preconditioned-sgd}). We prove that these algorithms, including the popular L-BFGS~\citep{liu1989limited} (the limited-memory version of BFGS), can attain global linear convergence with a constant batch-size. Our result is in contrast to previous work that shows the global linear convergence of stochastic BFGS algorithms by either using variance-reduction~\citep{kolte2015accelerating, lucchi2015variance, moritz2016linearly} or progressive batching strategies~\citep{bollapragada2018progressive}. 

Finally, in~\cref{sec:experiments}, we evaluate R-SSN and stochastic L-BFGS for binary classification. We use synthetic linearly-separable datasets and consider real datasets under a kernel mapping. We use automatic differentiation and truncated conjugate gradient~\citep{hestenes1952methods} to develop a ``Hessian-free'' implementation of R-SSN that allows computing the Newton update without additional memory overhead.\footnote{With code: \url{https://github.com/IssamLaradji/ssn}.} When interpolation holds, we observe that both R-SSN and stochastic L-BFGS result in faster convergence when compared to popular first-order methods~\citep{johnson2013accelerating, kingma2014adam, duchi2011adaptive}. Furthermore, a modest batch-growth strategy and stochastic line-search~\citep{vaswani2019painless} scheme ensure that R-SSN is computationally efficient and competitive with stochastic first-order methods and L-BFGS variants in terms of both the number of iterations and wall-clock time required for convergence. 
\section{Background}
\label{sec:background}
We consider unconstrained minimization of a finite sum, $\min_{\x\in\R^d} f(\x) \defeq \frac{1}{n}\sum_{i=1}^n f_i(\x)$, where $f$ and each $f_i$ are twice continuously differentiable. The regularized subsampled Newton method (R-SSN) takes a step in the \emph{subsampled Newton direction} at iteration $k$: 
\begin{equation}
\xkk = \xk - \etak \left[\shessf{\batchS_k}{\xk}\right]^{-1}\sgradf{\batchG_k}{\xk},
\label{eq:inexactnewtonupdate}
\end{equation}
where $\etak$ is the step size. The sets $\Gk$ ($b_{g_k} = |\batchG_k|$) and $\Sk$ ($b_{s_k} = |\batchS_k|$) are independent samples of indices uniformly chosen from $\{1,\dots,n\}$ without replacement. The subsampled gradient and the regularized subsampled Hessian are defined as
\begin{align}
\sgradf{\batchG_k}{\xk}  &= \frac{1}{b_{g_k}}\sum_{i \in \batchG_k} \sgradf{i}{\xk} ,\\
\label{eq:subsampled-hessian-with-tau}
\shessf{\batchS_k}{\xk}  &= \frac{1}{b_{s_k}} \sum_{i \in \batchS_k} \nabla^2f_i(\xk) + \tau I_d
\end{align}
where $\tau\geq0$ is the LM regularization~\citep{levenberg1944method,marquardt1963algorithm}. Both the subsampled gradient and Hessian are unbiased, implying that $\E_{\batchG_k}[\sgradf{\batchG_k}{\xk}] = \gradf{\xk}$ and $\E_{\batchS_k}[\shessf{\batchS_k}{\xk}] = \hessf{\xk} + \tau I_d$. 

Throughout this paper, $\norm{\cdot}$ denotes the $\ell_2$ norm of a vector or the spectral norm of a matrix. For all our convergence results, we make the standard assumptions
that $f$ is $\mu$-strongly convex and $L$-smooth. We denote $\kappa=L/\mu$ as the condition number of $f$, and $\xopt$ as the unique minimizer of $f$. Additionally, we assume that each component $f_i$ is $L_i$-smooth and $\mu_i$-strongly-convex. Although $\mu_i$ can be zero for some of the components, meaning they are only convex, the finite sum is strongly-convex with $\mu > 0$. Furthermore, we define $\bar{\mu} = \sum_{i = 1}^{n} \mu_i/n\leq \mu$ and $\bar{L} = \sum_{i = 1}^{n} L_i/n \geq L$ to be the average strong-convexity and smoothness constants of $f$. 

These assumptions imply that for any subsample $S$, the function $\sum_{i \in S} f_i / |S|$ is $L_{\batchS}$-smooth and $\mu_{\batchS}$-strongly-convex, hence the eigenvalues of the corresponding subsampled Hessian can be upper and lower-bounded~\citep{bollapragada2018exact, bollapragada2018progressive, moritz2016linearly}. In particular, if $\mutilde = \min_S {\mu_S} \geq 0$ and $\Ltilde = \max_{S} L_S$, then for any sample $\batchS$ and point $\x$, the \emph{regularized} subsampled Hessian $\shessf{\batchS}{\x}$ has eigenvalues bounded in the $[\mutilde + \tau, \Ltilde + \tau]$ range. The LM regularization thus ensures that $\shessf{\batchS}{\x}$ will always be positive definite, the subsampled Newton direction exists and is unique. 

In this work, we focus on models capable of \emph{interpolating} the training data. Typically, these models are highly expressive and are able to fit all the training data. In the finite sum setting, interpolation implies that if $\gradf{\xopt}=0$, then $\sgradf{i}{\x^*}=0$ for all $i$ ~\citep{vaswani2019painless, vaswani2018fast, bassily2018exponential, ma2018power}, meaning that all the individual functions $f_i$ are minimized at the optimal solution $\x^*$. For the smooth, strongly-convex finite sums we consider, interpolation implies the \emph{strong growth condition} (SGC)~\citep[Propositions 1, 2]{vaswani2018fast}, ~\citep{schmidt2013fast}: for some constant $\rho \geq 1$ and all $\x$,
\begin{equation*}
\E_i\normsq{\sgradf{i}{\x}} \leq \rho\normsq{\gradf{\x}} \tag{SGC}.
\end{equation*}
For example, if the training data spans the feature space, then the SGC is satisfied for models interpolating the data when using the squared loss (for regression) and the squared hinge loss (for classification). 
\section{Subsampled Newton Methods}
\label{sec:theory}
In this section, we present our main theoretical results for R-SSN. We characterize its global linear convergence (Section~\ref{sec:ssn-global}) and local quadratic convergence (Section~\ref{sec:ssn-local}) under interpolation. We use the notions of Q and R-convergence rates~\citep{ortega1970iterative,nocedal2006numerical} reviewed in Appendix~\ref{app:convergence-rates}. 

\subsection{Global convergence}
\label{sec:ssn-global}
In the following theorem, we show that for smooth and strongly-convex functions satisfying interpolation, R-SSN with an adaptive step-size and constant batch-sizes for both the subsampled gradient and Hessian results in linear convergence from an arbitrary initialization.
\begin{theorem}[Global linear convergence]
\label{thm:globalinexactnewton}
Under (a) $\mu$-strong-convexity, (b) $L$-smoothness of $f$, (c) $[\mutilde + \tau, \Ltilde + \tau]$-bounded eigenvalues of the subsampled Hessian and (d) $\rho$-SGC, the sequence $\{\xk\}_{k\geq0}$ generated by R-SSN with (\rnum{1}) step-size $\etak = \frac{\left(\mu_{\Sk}+\tau\right)^2}{L\left(\left(\mu_{\Sk}+\tau\right)+\left(L_{\Sk}+\tau\right)c_g\right)}$ and (\rnum{2}) constant batch-sizes $b_{s_k} = b_s$, $b_{g_k} = b_g$ converges to $\xopt$ at a Q-linear rate from an arbitrary initialization $\x_0$,
\begin{equation*}
\E[f(\x_T)] - f(\xopt) \leq \left(1-\alpha\right)^T (f(\x_0)-f(\xopt)) , \text{ where}
\end{equation*}
$\alpha=\min\Big\{\frac{(\bar{\mu}+\tau)^2}{2\kappa c_g(\Ltilde+\tau)}\,,\,\frac{(\bar{\mu}+\tau)}{2\kappa(\Ltilde+\tau)}\Big\}$ and  $c_g =\frac{(\rho-1) \, (n - b_g)}{(n-1) \, b_g}$.
\end{theorem}
The proof is given in Appendix~\ref{app:ssn-global}. While the dependence on $b_{g}$ is explicit, the dependence on $b_{s}$ is through the constant $\mutilde$; as $b_{s}$ tends to $n$, $\mu_{\Sk}$ tends to $\mu$, allowing R-SSN to use a larger step-size that results in faster convergence (since $\mu \geq \bar{\mu}$). If we set $b_{g} = b_{s} = n$ and $\tau = 0$, we obtain the rate $\E[f(\x_T)] - f(\xopt) \leq (1 - 1/(2\kappa^2) )^{T} \, [\E[f(\x_0)] - f(\xopt)]$ which matches the deterministic rate~\citep[Theorem 2]{karimireddy2018global} up to a constant factor of $2$. The minimum batch-size required to achieve the deterministic rate is $b_g\geq n/(1+(\mubar+\tau)(n-1)/(\rho-1))$. Similar to SGD~\citep{vaswani2018fast, schmidt2013fast}, the interpolation condition allows R-SSN with a constant batch-size to obtain Q-linear convergence. In the absence of interpolation, SSN has only been shown to achieve an R-linear rate by increasing the batch size geometrically for the subsampled gradient~\citep{bollapragada2018exact}. Next, we analyze the convergence properties of R-SSN in a local neighbourhood of the solution. 

\subsection{Local convergence}
\label{sec:ssn-local}
To analyze the local convergence of R-SSN, we make additional assumptions. Similar to other local convergence results~\citep{roosta2016subb}, we assume that the Hessian is $M$-Lipschitz continuous, implying that for all $w, v\in\R^d$, $\norm{\hessf{w}-\hessf{v}} \leq M \norm{w-v}.$ We also assume that the \emph{moments of the iterates are bounded}, implying that for all iterates $\xk$, there exists a constant $\gamma$ such that $\E \, [ \normsq{\xk - \xopt} \, ] \leq \gamma \; [ \, \E\norm{\xk-\xopt} \, ]^2$ where the expectation is taken over the entire history up until step $k$. If the iterates lie within a bounded set, then the above assumption holds for some finite $\gamma$~\citep{bollapragada2018exact, berahas2017investigation, harikandeh2015stopwasting}. We also assume that the sample variance of the subsampled Hessian is bounded, namely $
\frac{1}{n-1}\sum_{i=1}^n[\normsq{\nabla^2f_{i}(\x)-\hessf{\x}}]\leq\sigma^2$ for some $\sigma>0$. 

\begin{theorem}[Local convergence]
\label{thm:inexactnewtonlocal}
Under the same assumptions (a) - (d) of Theorem~\ref{thm:globalinexactnewton} along with (e) $M$-Lipschitz continuity of the Hessian, (f) $\gamma$-bounded moments of iterates, and (g) $\sigma$-bounded variance of the subsampled Hessian, the sequence $\{\xk\}_{k\geq 0}$ generated by R-SSN with (\rnum{1}) unit step-size $\eta_k = 1$ and (\rnum{2}) growing batch-sizes satisfying
\begin{equation*}
b_{g_k} \geq \frac{n}{(\frac{n-1}{\rho-1})\normsq{\gradf{\x_k}}+1} \text{, } b_{s_k}\geq\frac{n}{\frac{n}{\sigma^2}\norm{\gradf{\x_k}}+1} 
\end{equation*}
converges to $\xopt$ at a linear-quadratic rate
\begin{align*}
\mathbb{E} \norm{\x_{k+1}-\x^*} \leq & \frac{\gamma (M + 2 L + 2 L^2 )}{2 (\mutilde+\tau)} \, \left(\E \norm{\x_k-\x^*} \right)^2 \\
& + \vspace{2ex} \frac{\tau}{\mutilde+\tau}\E\norm{\xk-\xopt}
\end{align*}
in a local neighbourhood of the solution $\norm{\x_0-\x^*}\leq \frac{2(\tilde{\mu}+\tau)}{\gamma(M+2L+2L^2)}$. Furthermore, when $\tau=0$ and $\mutilde>0$, R-SSN can achieve local quadratic convergence
\begin{equation*}
\E \norm{\x_{k+1}-\x^*} \leq \gamma \Big(\frac{M+2L+2L^2}{2\tilde{\mu}}\Big)\left(\E\norm{\x_k-\x^*}\right)^2.
\end{equation*}
\end{theorem}
The proof is provided in Appendix~\ref{app:ssn-local}. Note that the constant $\frac{\tau}{\mutilde+\tau}$ on the linear term is less than one and hence contraction is guaranteed, while the constant on the quadratic term is not required to be less than one to guarantee convergence~\citep{nesterov2018lectures}. 

This theorem states that if we progressively increase the batch size for both the subsampled Hessian and the subsampled gradient, then we can obtain a local linear-quadratic convergence rate. For inexact Newton methods to obtain quadratic convergence, it has been shown that the error term needs to decrease as a quadratic function of the gradient~\citep{dembo1982inexact}, which is not provided by SGC, and hence the need for additional techniques such as batch growing. The required geometric growth rate for $\batchG_k$ is the same as that of SGD to obtain linear convergence without variance-reduction or interpolation~\citep{de2016big, friedlander2012hybrid}. Note that the proof can be easily modified to obtain a slightly worse linear-quadratic rate when using a subsampled Hessian with a constant batch-size. In addition, following ~\citet{ye2017approximate}, we can relax the Lipschitz Hessian assumption and obtain a slower superlinear convergence. Unlike the explicit quadratic rate above, in the absence of interpolation, SSN without regularization has only been shown to achieve an asymptotic superlinear rate~\citep{bollapragada2018exact}. Moreover, in this case, $b_{g_k}$ needs to increase at a rate that is faster than geometric, considerably restricting its practical applicability.

The following corollary (proved in Appendix~\ref{sec:decreasing-tau}) states that if we decay the LM-regularization sequence $\tau_k$ proportional to the gradient norm, R-SSN can achieve quadratic convergence for strongly-convex functions. In fact, this decay rate is inversely proportional to the growth of the batch size for the subsampled Hessian, indicating that larger batch-sizes require smaller regularization. This relationship between the regularization and sample size is consistent with the observations of~\citet{ye2017approximate}.
\begin{corollary}
\label{corollary:decreasing-tau}
Under the same assumptions as Theorem~\ref{thm:inexactnewtonlocal}, if we decrease the regularization term according to $\tau_k\leq \norm{\gradf{\xk}}$, R-SSN can achieve local quadratic convergence with $\norm{\x_0-\xopt}\leq \frac{2 (\mutilde+\tau_k)}{\gamma (M + 4L + 2 L^2 )}$:
\begin{equation*}
 \mathbb{E} \norm{\x_{k+1}-\x^*} \leq \frac{\gamma (M + 4L + 2 L^2 )}{2 (\mutilde_k+\tau_k)} \, \left(\E \norm{\x_k-\x^*} \right)^2,
\end{equation*}
where $\mutilde_k$ is the minimum eigenvalue of $\nabla^2f_{\batchS_k}(\xk)$ over all batches of size $|\batchS_k|$.
\end{corollary}

Since $\mutilde\leq\mutilde_k$, this gives a tighter guarantee on the decrease in suboptimality at every iteration as the batch size grows and $\tau_k$ decreases. On the other hand, if we make a stronger growth assumption on the stochastic gradients such that $\mathbb{E}_i\normsq{\sgradf{i}{\x}}\leq\rho\norm{\gradf{\x}}^4$, then R-SSN can achieve local quadratic convergence using only a \emph{constant} batch-size for the subsampled gradient and the same growth rate for the subsampled Hessian. We state this result as Corollary~\ref{cor:SSN-local-stronger-SGC}, but we acknowledge that this assumption might be too restrictive to be useful in practice. 

\section{R-SSN under Self-Concordance}
\label{sec:theory-self-concordance}
We now analyze the convergence of R-SSN for the class of strictly-convex self-concordant functions. A function $f$ is called \textit{standard self-concordant}\footnote{If $f$ is self-concordant with parameter $C\geq 0$, re-scaling it by a factor $C^2/4$ makes it standard self-concordant~\citep[Corollary 5.1.3]{nesterov2018lectures}, and hence we consider only standard self-concordant functions in this work.}  if the following inequality holds for any $\x$ and direction $u \in \R^d$,
\begin{align*}
\vert u\transpose \, (\tensf{\x}[u]) \, u \vert \leq 2\, \indnorm{u}{\hessf{\x}}^{3/2} \; .
\end{align*}
Here, $\tensf{\x}$ denotes the third derivative tensor at $\x$ and $\tensf{\x}[u]$ represents the matrix corresponding to the projection of this tensor on vector $u$. Functions in this (generalized) class include common losses such as the Huber loss~\citep{marteau2019beyond}, smoothed hinge loss~\citep{zhang2015disco}, and the logistic loss~\citep{bach2010self}. Self-concordant functions enable an affine-invariant analysis of Newton's method and have convergence rates similar to that of smooth and strongly-convex functions, but without problem-dependent constants~\citep{nesterov2018lectures, boyd2004convex}. To characterize the convergence rate of R-SSN under self-concordance, we define the regularized Newton decrement at $\x$ as:
\begin{equation*}
\nd{\x} \defeq \langle\gradf{\x}\,,\,\left[\hessf{\x}+\tau I_d\right]^{-1}\gradf{\x}\rangle^{1/2}    
\end{equation*}
When $\tau = 0$, this corresponds to the standard definition of the Newton decrement~\citep{nesterov2018lectures}. If we denote the standard Newton decrement as $\lambda^0(\x)$, then the regularized version can be bounded as $\nd{\x} \leq \lambda^0(\x)$. We also introduce the regularized stochastic Newton decrement as
\begin{equation*}
\tilde{\lambda}_{i,j}(\x)\defeq\langle\sgradf{i}{\x}\,,\,\left[\shessf{j}{\x}\right]^{-1}\sgradf{i}{\x}\rangle^{1/2}    
\end{equation*}
for independent samples $i$ and $j$. In Proposition~\ref{prop:newton-decr-sgc} (Appendix~\ref{app:common-lemmas}), we show that if $f$ satisfies the SGC and the subsampled Hessian has bounded eigenvalues, then $f$ satisfies a similar growth condition for the regularized stochastic Newton decrement: for all $\x$ and $j$, there exists $\rhond \geq 1$ such that 
\begin{align*}
\E_{i} \left[\tilde{\lambda}_{i,j}^2(\x) \right] \leq \rhond \, \lambda^2(\x) \tag{Newton decrement SGC}.
\end{align*}
For ease of notation, we use $\tilde{\lambda}_k $ to denote the regularized stochastic Newton decrement computed using samples $\batchG_k$ and $\batchS_k$. We now consider an update step similar to (\ref{eq:inexactnewtonupdate}) but with a step-size adjusted to $\tilde{\lambda}_k$,
\begin{equation}
\label{eq:update-all-stochastic-nd}
\xkk = \xk - \frac{c \; \eta}{1+\eta \; \tilde{\lambda}_k}\left[\shessf{\batchS_k}{\xk}\right]^{-1}\sgradf{\batchG_k}{\xk}
\end{equation}
where $c,\eta\in(0,1]$. For the next theorem, we assume the iterates lie in a bounded set such that $\norm{\xk-\xopt}\leq D$ for all $k$~\citep{harikandeh2015stopwasting}.
\begin{theorem}
\label{thm:full-stochastic-self-concord}
Under (a) self-concordance (b) $L$-smoothness of $f$ (c) $[\mutilde + \tau,\Ltilde + \tau]$-bounded eigenvalues of the subsampled Hessian (d) $\rhond=\frac{\rho L}{\mutilde+\tau}$-Newton decrement SGC and (e) bounded iterates, the sequence $\{\xk\}_{k \in [0,m]}$ generated by R-SSN in $(\ref{eq:update-all-stochastic-nd})$ with (i) $c=\sqrt{\frac{\mutilde+\tau}{L}}$, $\eta\in\Big(0,\frac{c}{\rhond(1+\Ltilde D/(\mutilde+\tau))}\Big]$ and (ii) constant batch sizes $\geq 1$ converges to $\xopt$ from an arbitrary initialization $\x_0$ at a rate characterized by
\begin{equation*}
\expectationof{f(\xkk)} \leq f(\xk) - \eta\,\delta \, \omegaof{\lambda_k}.
\end{equation*}
Here $\delta\in(0,1]$ and the univariate function $\omega$ is defined as $\omegaof{t} = t - \ln(1+t)$. Furthermore, in the local neighbourhood where $\lambda_m \leq 1/6$, the sequence $\{\xk\}_{k \geq m}$ converges to $\xopt$ at a Q-linear rate,
\begin{equation*}
\expectationof{f(\x_T)} - f(\xopt) \leq \Big(1-\frac{\eta\delta}{1.26}\Big)^{T-m}(\expectationof{f(\x_m)}-f(\xopt)).
\end{equation*}
\end{theorem}
The proof is given in Appendix~\ref{app:full-stochastic-self-concord}. The above result gives an insight into the convergence properties of R-SSN for loss functions that are self-concordant but not necessarily strongly-convex. Note that the analysis of the deterministic Newton's method for self-concordant functions yields a problem-independent local quadratic convergence rate, meaning it does not rely on the condition number of $f$. However, it is non-trivial to improve the above analysis and obtain a similar affine-invariant result. Similarly, the algorithm parameters for the inexact Newton method in~\citet{zhang2015disco} depend on the condition number. 
\section{Stochastic BFGS}
\label{sec:preconditioned-sgd}
Consider the stochastic BFGS update, 
\begin{equation}
\label{eq:precondsgdupdate}
\xkk = \xk -\etak \, \Hk \, \sgradf{\batchG_k}{\xk}
\end{equation}
where $\Hk$ is a positive definite matrix constructed to approximate the inverse Hessian $\left[\hessf{\xk}\right]^{-1}$. As in previous works~\citep{bollapragada2018progressive, moritz2016linearly, lucchi2015variance, sunehag2009variable}, we assume that $\Hk$ has bounded eigenvalues such that $\lambda_1I \preceq \Hk \preceq\lambda_dI$. We now show that under SGC, stochastic BFGS achieves global linear convergence with a constant step-size. 
\begin{theorem}
\label{thm:lbfgs}
Under (a) $\mu$-strong-convexity, (b) $L$-smoothness of $f$, (c) $\left[\lambda_1, \lambda_d \right]$-bounded eigenvalues of $\Hk$ and (d) $\rho$-SGC, the sequence $\{\xk\}_{k\geq 0}$ generated by stochastic BFGS with (i) a constant step-size $\etak = \eta = \frac{\lambda_1}{c_g L\lambda_d^2}$ and (ii) and a constant batch size $b_{g_k} = b_g$ converges to $\xopt$ at a linear rate from an arbitrary initialization $\x_0$,
\begin{equation*}
\E[f(\x_{T})]-f(\xopt) \leq  \Big(1-\frac{\mu \lambda_1^2}{c_g \, L \lambda_d^2}\Big)^{T} (f(\x_0)-f(\xopt)) 
\end{equation*}
where $c_g = \frac{(n - b_g) \, (\rho - 1)}{(n-1) \, b_g} + 1$.
\end{theorem}
The proof is given in Appendix~\ref{sec:proof-lbfgs}. The strong-convexity assumption can be relaxed to the Polyak-{\L}ojasiewicz inequality~\citep{polyak1963gradient, karimi2016linear} while giving the same rate up to a constant factor. In the absence of interpolation, global linear convergence can only be obtained by using either variance-reduction techniques~\citep{moritz2016linearly, lucchi2015variance} or progressive batching~\citep{bollapragada2018progressive}. Similar to these works, our analysis for the stochastic BFGS method applies to preconditioned SGD where $\Hk$ can be any positive-definite preconditioner.  
\section{Experiments}
\label{sec:experiments}
We verify our theoretical results on a binary classification task on both synthetic and real datasets. We evaluate two variants of R-SSN: \texttt{R-SSN-const} that uses a constant batch-size and \texttt{R-SSN-grow} where we grow the batch size geometrically (by a constant multiplicative factor of $1.01$ in every iteration~\citep{friedlander2012hybrid}). Note that although our theoretical analysis of R-SSN requires independent batches for the subsampled gradient and Hessian, we use the same batch for both variants of R-SSN and observe that this does not adversely affect empirical performance. We use truncated conjugate gradient (CG)~\citep{hestenes1952methods} to solve for the (subsampled) Newton direction in every iteration. For each experiment, we choose the LM regularization $\tau$ via a grid search. For \texttt{R-SSN-grow}, following Corollary~\ref{corollary:decreasing-tau}, starting from the LM regularization selected by grid search, we progressively decrease $\tau$ in the same way as we increase the batch size. We evaluate stochastic L-BFGS (\texttt{sLBFGS}) with a ``memory'' of $10$. For \texttt{sLBFGS}, we use the same minibatch to compute the difference in the (subsampled) gradients to be used in the inverse Hessian approximation (this corresponds to the ``full'' overlap setting in~\citet{bollapragada2018progressive}), to which we add a small regularization to ensure positive-definiteness. For both R-SSN and \texttt{sLBFGS}, we use the stochastic line-search from~\citet{vaswani2019painless} to set the step size.  

We compare the proposed algorithms against common first-order methods: SGD, SVRG~\citep{johnson2013accelerating}, Adam~\citep{kingma2014adam} and Adagrad~\citep{duchi2011adaptive}. For all experiments, we use an (initial) batch size of $b = 100$ and run each algorithm for $200$ epochs. Here, an epoch is defined as one full pass over the dataset and does not include additional function evaluations from the line-search or CG. Subsequently, in Figure~\ref{fig:exp-kernel-main}, we plot the mean wall-clock time per epoch that takes these additional computations into account. For SGD, we compare against its generic version and with Polyak acceleration~\citep{polyak1964some}, where in both cases the step size is chosen via stochastic line-search~\citep{vaswani2019painless} with the same hyperparameters, and the acceleration hyperparameter is chosen via grid search. For SVRG, the step size is chosen via 3-fold cross validation on the training set, and we set the number of inner iterations per outer full-gradient evaluation to $n/b$. We use the default hyperparameters for Adam and Adagrad for their adaptive properties. All results are averaged across $5$ runs. 

\textbf{Synthetic datasets:}
\begin{figure*}[!t]
\centering
\includegraphics[scale=0.31]{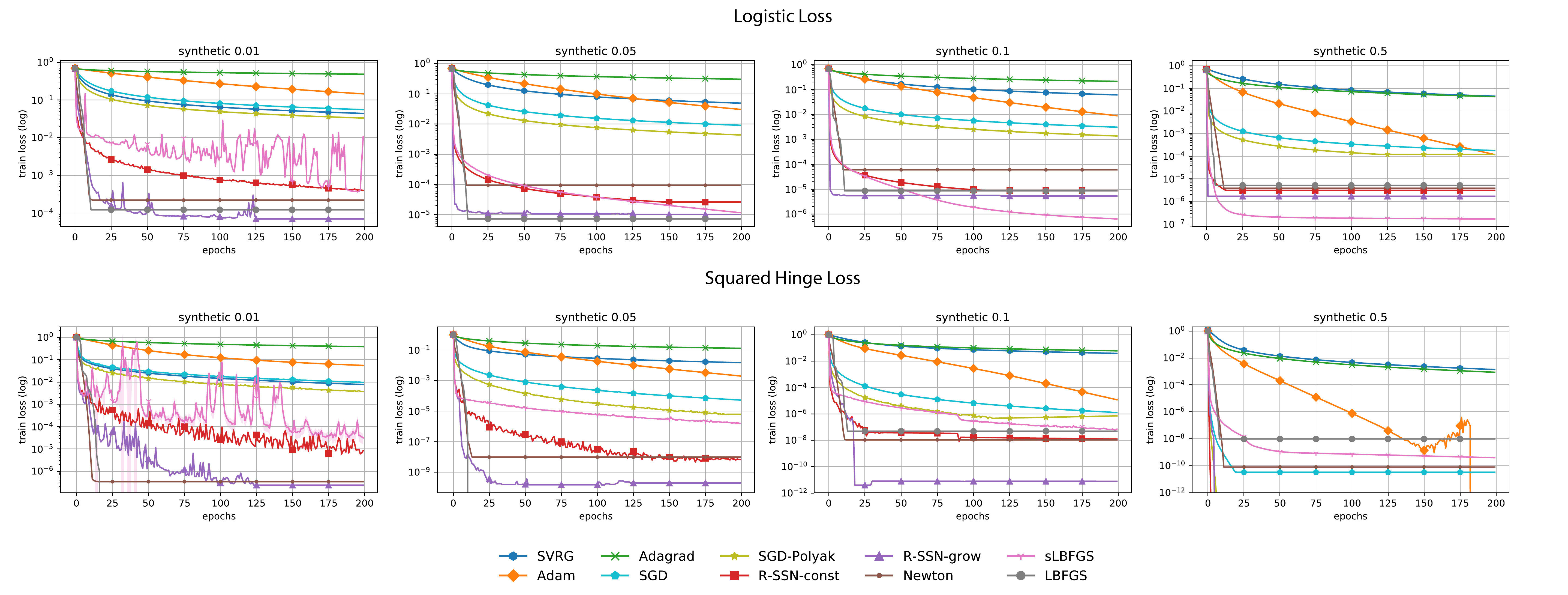}
\caption{Comparison of R-SSN variants and stochastic L-BFGS against first order methods on synthetic data where interpolation is satisfied. Both R-SSN and stochastic L-BFGS outperform first order methods. From left to right we show the results for datasets with margins \texttt{[0.01, 0.05, 0.1, 0.5]} under two different losses.}
\label{fig:exp-syn}
\end{figure*}
We first evaluate the algorithms on binary classification using synthetic and linearly-separable datasets with varying margins. Linear separability ensures the interpolation condition will hold. For each margin, we generate a dataset with $10$k examples with $d=20$ features and binary labels. For these datasets, we also compare against unregularized Newton and L-BFGS, both using the full-batch (hence deterministic). For L-BFGS, we use the PyTorch~\citep{paszke2017automatic} implementation with line search and an initial step size of $0.9$. 

In Figure~\ref{fig:exp-syn}, we show the training loss for the logistic loss (row 1) and the squared hinge loss (row 2). We observe that by incorporating second order information, R-SSN can converge much faster than first order methods. In addition, global linear convergence can be obtained using only constant batch-sizes, verifying Theorem~\ref{thm:globalinexactnewton}. Furthermore, by growing the batch size, R-SSN performs similar to the deterministic Newton method, verifying the local quadratic convergence of Theorem~\ref{thm:inexactnewtonlocal}. Although our theory only guarantees linear convergence globally (instead of superlinear) for stochastic L-BFGS under interpolation, these methods can be much faster than first order methods empirically. Finally, as we increase the margin (left to right of Figure~\ref{fig:exp-syn}), the theoretical rate under interpolation improves~\citep{vaswani2018fast}, resulting in faster and more stable convergence for the proposed algorithms. 

\textbf{Real datasets:}
\begin{figure*}[!t]
\centering
\includegraphics[scale=0.3]{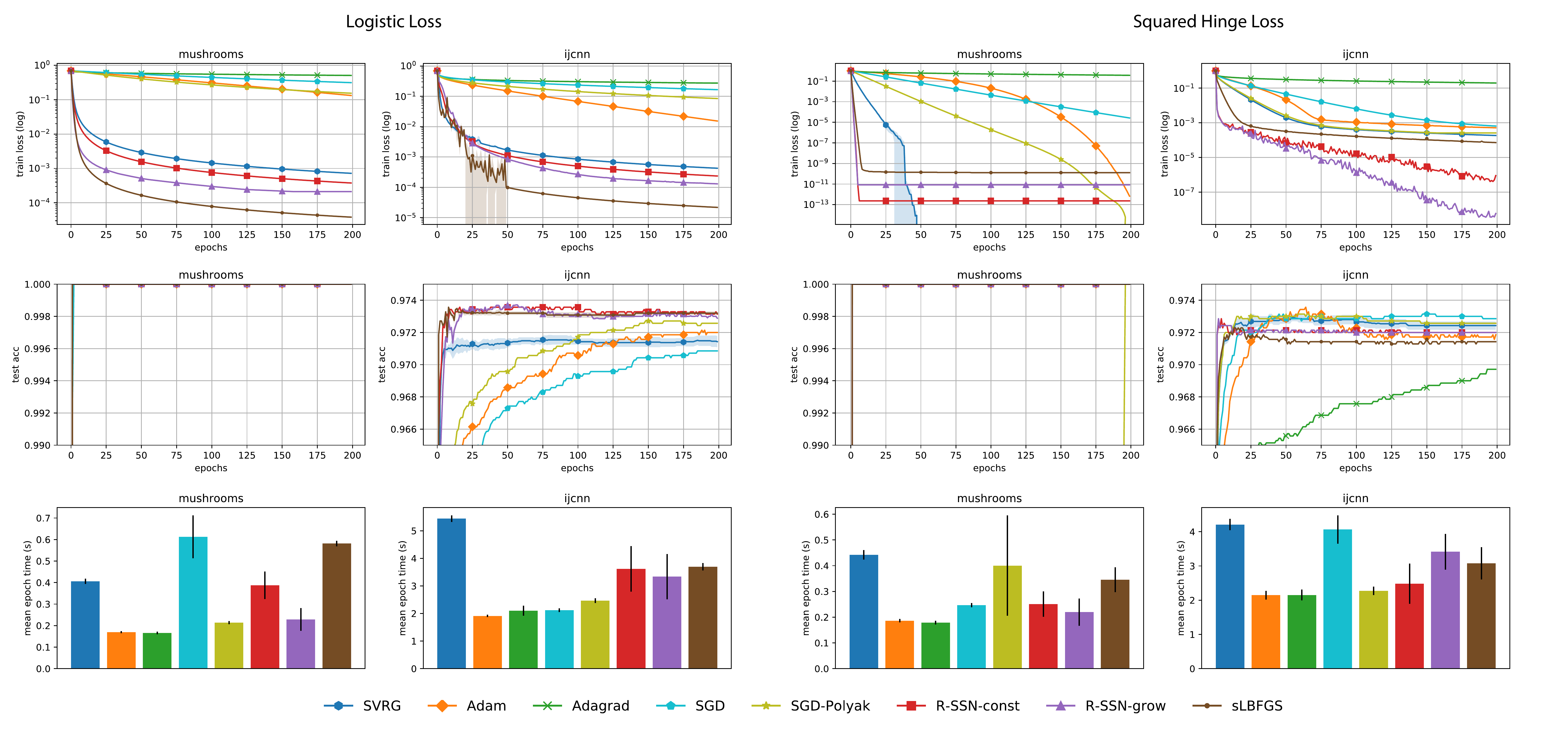}
\caption{Comparison of R-SSN variants and stochastic L-BFGS against first order methods on real datasets. Under the RBF kernel mapping, the \texttt{mushrooms} dataset satisfies interpolation, and the R-SSN variants and \texttt{sLBFGS} perform best in this setting. When interpolation is not satisfied (\texttt{ijcnn}), these methods are still competitive.}
\label{fig:exp-kernel-main}
\end{figure*}
We also consider real datasets \texttt{mushrooms}, \texttt{rcv1}, and \texttt{ijcnn} from the LIBSVM repository \citep{chang2011libsvm}, and use an $80:20$ split for the training and test set, respectively. We fit a linear model under the radial basis function (RBF) kernel. The kernel mapping results in the effective dimension being equal to the number of points in the dataset. This results in problem dimensions of $6.5$k, $20$k and $28$k for \texttt{mushrooms}, \texttt{rcv1}, and \texttt{ijcnn} respectively. The RBF-kernel bandwidths are chosen via grid search using $10$-fold cross validation on the training split following \citet{vaswani2019painless}. Note that the \texttt{mushrooms} dataset is linearly separable under the chosen kernel mapping and thus satisfies the interpolation condition. For these datasets, we limit the maximum batch size to $8192$ for \texttt{SSN-grow} to alleviate the computation and memory overhead. We show the training loss, test accuracy, as well as the mean wall-clock time per epoch. For this set of experiments, the stochastic line-search procedure~\citep{vaswani2019painless} used in conjunction with \texttt{sLBFGS} can lead to a large number of backtracking iterations, resulting in a high wall-clock time per epoch. To overcome this issue, we use a constant step-size variant of \texttt{sLBFGS} and perform a grid search over $[10^{-4}, 1]$ to select the best step-size for each experiment. In the interest of space, we refer the reader to Appendix~\ref{app:additional-exps} for results on the \texttt{rcv1} dataset. 

In the first row of Figure~\ref{fig:exp-kernel-main}, we observe when interpolation is satisfied (\texttt{mushrooms}), the R-SSN variants and \texttt{sLBFGS} outperform all other methods in terms of training loss convergence. When interpolation is not satisfied (\texttt{ijcnn}), the stochastic second-order methods are still competitive with the best performing method. In the second row, we observe that despite the fast convergence of these methods, generalization performance does not deteriorate regardless of whether interpolation is satisfied. Furthermore, since all our experiments are run on the GPU, we take advantage of parallelization and amortize the runtime of the proposed methods. This is reflected in the third row, where we plot the mean per-epoch wall-clock time for all methods. Also note that \texttt{sLBFGS} is competitive with R-SSN in terms of the number of iterations, but has a higher per-epoch cost on average. 
\section{Related Work}
\label{sec:related-work}
Besides the works discussed earlier, there have been many recent works analyzing subsampled Newton methods~\citep{byrd2011use, byrd2012sample, milzarek2018stochastic, li2019subsampled}. In particular, \citet{roosta2016suba,roosta2016subb} derive probabilistic global and local convergence rates. Similarly, \citet{xu2016sub} achieve a local linear-quadratic convergence in high probability by using non-uniform sampling to select the batch. \citet{erdogdu2015convergence} propose the NewSamp method that subsamples the Hessian followed by a regularized truncated SVD to form the inverse Hessian approximation. \citet{ye2017approximate} provide a unified analysis of the local convergence properties of well-known variants of subsampled Newton's method where the gradient is assumed to be exact. For nonconvex objectives, \citet{bergou2018subsampling} show second order results for the global convergence of subsampled Newton's method with line search.  

Recently, stochastic optimization under an interpolation condition has been analyzed by several authors. Both \citet{ma2018power} and \citet{vaswani2018fast} show that under mild conditions, SGD with constant step-size achieves linear convergence for strongly-convex functions. Additionally, it can match the deterministic rates in the convex~\citep{vaswani2018fast, schmidt2013fast, cevher2018linear} and non-convex~\citep{bassily2018exponential, vaswani2018fast} settings. The interpolation assumption also allows for momentum-type methods to achieve the accelerated rates of convergence for least-squares~\citep{gower2018accelerated} and more generally in convex settings~\citep{vaswani2018fast, liu2018mass}. Although the step size in these settings depends on unknown quantities, it has been recently shown that stochastic line-search methods based on the Armijo condition can be used to automatically set the step size and still achieve fast convergence rates~\citep{vaswani2019painless}. 
In contrast to previous works on stochastic second-order methods, we consider fully stochastic R-SSN and show global Q-linear convergence and local Q-quadratic convergence in expectation for SSN for strongly-convex objectives under the interpolation condition.

The class of self-concordant functions was first introduced by \citet{nesterov1994interior} and has been extended to include common losses such as (multi-class) logistic regression~\citep{bach2010self}. \citet{zhang2015disco} propose a distributed inexact Newton's method and provide convergence analysis for self-concordant functions. \citet{mokhtari2016adaptive, eisen2017large} introduce the Ada Newton method that computes a geometrically increasing batch-size based on a targeted statistical accuracy and show local quadratic convergence for self-concordant functions with high probability. However, none of these works analyze stochastic Newton methods for self-concordant functions. 

Quasi-Newton methods that approximate the Hessian are popular in practice due to their cheaper iteration costs compared to full Newton methods while having superlinear convergence~\citep{dennis1974characterization}. Stochastic variants of these methods have been proposed to further reduce the computation overhead from a large training set~\citep{schraudolph2007stochastic,zhou2017stochastic, liu2018acceleration, berahas2019quasi, gao2019quasi}. Combined with variance-reduction techniques from popular first-order methods, stochastic BFGS-type methods can achieve global linear convergence~\citep{kolte2015accelerating, lucchi2015variance, moritz2016linearly}. \citet{bollapragada2018progressive} use progressive batching with stochastic L-BFGS to show linear convergence for strongly-convex functions and sublinear rate for convex functions. \citet{mokhtari2018iqn} use memory to reduce the variance for stochastic BFGS and achieve a local superlinear convergence rate. \citet{kelley2002brief}'s implicit filtering method can achieve superlinear convergence using a deterministic and finite-difference-based BFGS for minimizing objectives whose noise decreases as we approach the global minimum, similar to the interpolation setting. Our work fills in the gap between stochastic quasi-Newton methods and the interpolation condition.
\vspace{-1.5mm}
\section{Conclusion}
\vspace{-1mm}
\label{sec:conclusion}
We showed that the regularized subsampled Newton method with a constant batch-size achieves linear convergence rates when minimizing smooth functions that are strongly-convex or self-concordant under the interpolation setting. We also showed that interpolation enables stochastic BFGS-type methods to converge linearly. We validated our theoretical claims via experiments using kernel functions and demonstrated the fast convergence of the proposed methods. Our theoretical and empirical results show the potential for training large over-parameterized models that satisfy the interpolation property. For future work, we aim to investigate ways to handle non-convex losses and further scale these methods to optimize over millions of parameters.

\clearpage
\section{Acknowledgements}
We would like to thank Raghu Bollapragada, Nicolas Le Roux, Frederik Kunstner, Chris Liaw and Aaron Mishkin for helpful discussions. This research was partially supported by the Canada CIFAR AI Chair Program, by the NSERC Discovery Grants RGPIN-2017-06936 and 2015-06068, an IVADO postdoctoral scholarship, and an NSERC Canada Graduate Scholarships-Master’s. Simon Lacoste-Julien is a CIFAR Associate Fellow in the Learning in Machines \& Brains program and a Canada CIFAR AI Chair.
\bibliographystyle{plainnat}
\bibliography{ref}

\clearpage
\onecolumn
\appendix
\section{Proof of Theorem~\ref{thm:globalinexactnewton}}
\label{app:ssn-global}
\begin{proof}
This analysis closely follows the proof of Theorem 2.2 in \citet{bollapragada2018exact}. By the $L$-smoothness assumption,
\begin{align*}
& f(\xkk) \leq f(\xk) + \inner{\gradf{\xk}}{\xkk-\xk}+\frac{L}{2}\normsq{\xkk-\xk} \\
& = f(\xk)-\etak\inner{\gradf{\xk}}{\left[\shessf{\Sk}{\xk}\right]^{-1}\sgradf{\Gk}{\xk}}+\frac{L}{2}\etak^2\normsq{\left[\shessf{\Sk}{\xk}\right]^{-1}\sgradf{\Gk}{\xk}}. \tag{Update step} \\
\intertext{Since $\Sk$ and $\Gk$ are independent samples, we can fix the Hessian sample $\Sk$ and take expectation with respect to the unbiased gradient sample $\Gk$,}
& \E_{\Gk} [f(\xkk)] \leq f(\xk)-\etak\inner{\gradf{\xk}}{\left[\shessf{\Sk}{\xk}\right]^{-1}\gradf{\xk}}+\frac{L}{2}\etak^2\,\underbrace{\E_{\Gk}\normsq{\left[\shessf{\Sk}{\xk}\right]^{-1}\sgradf{\Gk}{\xk}}}_{\defeq\,P}.
\end{align*}
Using the fact $\E\normsq{x} = \E\normsq{x-\E x}+\normsq{\E x}$ with $x = \left[\shessf{\Sk}{\xk}\right]^{-1}\sgradf{\Gk}{\xk}$, we can bound the last term as
\begin{align*}
P & = \E_{\Gk} \left[\normsq{\left[\shessf{\Sk}{\xk}\right]^{-1}\sgradf{\Gk}{\xk} - \E_{\Gk}\left[\left[\shessf{\Sk}{\xk}\right]^{-1}\sgradf{\Gk}{\xk}\right]} \right]  + \normsq{\E_{\Gk}[\left[\shessf{\Sk}{\xk}\right]^{-1}\sgradf{\Gk}{\xk}]} \\
& = \E_{\Gk}\left[\normsq{\left[\shessf{\Sk}{\xk}\right]^{-1} \left[\sgradf{\Gk}{\xk} - \gradf{\xk} \right]}\right] + \normsq{\left[\shessf{\Sk}{\xk}\right]^{-1}\gradf{\xk}}. \tag{Again, by unbiasedness and independent batches} \\
\implies P &\leq \frac{1}{\left(\mu_{\Sk}+\tau\right)^2}\E_{\Gk} \left[\normsq{\sgradf{\Gk}{\xk} - \gradf{\xk}}\right] + \normsq{\left[\shessf{\Sk}{\xk}\right]^{-1}\gradf{\xk}}. \tag{Since $\shessf{\Sk}{\xk} \succeq ( \mu_{\Sk}+\tau) I_d$} \\
\intertext{Now we use Lemma~\ref{lemma:grad-batch-sgc} in Appendix~\ref{app:common-lemmas} to obtain }
& \E_{\Gk}\normsq{\sgradf{\Gk}{\x_k}-\gradf{\x_k}} \leq \frac{n - b_{g_k}}{(n-1) \, b_{g_k}}(\rho-1)\normsq{\gradf{\x_k}},
\end{align*}
which implies
\begin{equation*}
P \leq \frac{\rho-1}{\left(\mu_{\Sk}+\tau\right)^2} \frac{n - b_{g_k}}{(n-1) \, b_{g_k}}\normsq{\gradf{\x_k}} + \normsq{\left[\shessf{\Sk}{\xk}\right]^{-1}\gradf{\xk}}. 
\end{equation*}
From the above relations, we have
\begin{align*}
\E_{\Gk}[f(\xkk)] & \leq f(\xk) +  \frac{L\etak^2}{2} \frac{\rho-1}{\left(\mu_{\Sk}+\tau\right)^2} \frac{n - b_{g_k}}{(n-1) \, b_{g_k}}\normsq{\gradf{\x_k}} \\ 
& \phantom{aaaaaaaa}-\underbrace{\etak \inner{\gradf{\xk}}{\left[\shessf{\Sk}{\xk}\right]^{-1}\gradf{\xk}} +  \frac{L\etak^2}{2}\normsq{\left[\shessf{\Sk}{\xk}\right]^{-1}\gradf{\xk}}}_{\defeq\, Q} .
\end{align*}
Expanding $\normsq{\left[\shessf{\Sk}{\xk}\right]^{-1}\gradf{\xk}}$ and decomposing $\left[\shessf{\Sk}{\xk}\right]^{-1}$ gives us
\begin{align*}
Q& = -\etak\inner{\gradf{\xk}}{\left(\left[\shessf{\Sk}{\xk}\right]^{-1}-\frac{L\etak}{2}\left[\shessf{\Sk}{\xk}\right]^{-2}\right)\gradf{\xk}} \\
& = -\etak\inner{\left[\shessf{\Sk}{\xk}\right]^{-1/2}\gradf{\xk}}{\left(I-\frac{L\etak}{2}\left[\shessf{\Sk}{\xk}\right]^{-1}\right)\left[\shessf{\Sk}{\xk}\right]^{-1/2}\gradf{\xk}} \\
&\leq -\etak\left(1-\frac{L\etak}{2\left(\mu_{\Sk}+\tau\right)}\right)\normsq{\left[\shessf{\Sk}{\xk}\right]^{-1/2}\gradf{\xk}}\\
&\leq -\frac{\etak}{\left(L_{\Sk}+\tau\right)}\left(1-\frac{L\etak}{2\left(\mu_{\Sk}+\tau\right)}\right)\normsq{\gradf{\xk}},
\end{align*}
where the second last inequality requires $\left(I-\frac{L\etak}{2}\left[\shessf{\Sk}{\xk}\right]^{-1}\right)$ to be positive definite, which is true for step-sizes satisfying
\begin{equation}
\label{eq:eta_constraint}
\etak\leq\frac{2\left(\mu_{\Sk}+\tau\right)}{L}.
\end{equation}
Then we have
\begin{align*}
\E_{\Gk}[f(\xkk)] &\leq f(\xk) - \frac{\etak}{\left(L_{\Sk}+\tau\right)} \left(1-\frac{L\etak}{2\left(\mu_{\Sk}+\tau\right)}\right)\normsq{\gradf{\xk}} + \frac{L\etak^2}{2} \frac{\rho-1}{\left(\mu_{\Sk}+\tau\right)^2} \frac{n - b_{g_k}}{(n-1) \, b_{g_k}}\normsq{\gradf{\x_k}} \\
&= f(\xk) - \left[\frac{\etak}{\left(L_{\Sk}+\tau\right)} \left(1-\frac{L\etak}{2\left(\mu_{\Sk}+\tau\right)}\right)-\frac{L\etak^2}{2} \frac{\rho-1}{\left(\mu_{\Sk}+\tau\right)^2} \frac{n - b_{g_k}}{(n-1) \, b_{g_k}}\right]\normsq{\gradf{\xk}}.
\end{align*}
Subtract $f(\xopt)$ from both sides and using the fact that strong convexity implies $\normsq{\gradf{x}}\geq2\mu(f(x)-f(x^*))$ for all $x$, the above bound becomes
\begin{align*}
\E_{\Gk}[f(\xkk)] - f(\xopt) &\leq f(\xk) - f(\xopt) \\
& \phantom{aaaa}-  2\mu \left[\frac{\etak}{\left(L_{\Sk}+\tau\right)}\left(1-\frac{L\etak}{2\left(\mu_{\Sk}+\tau\right)}\right) - \frac{L\etak^2}{2} \frac{\rho-1}{\left(\mu_{\Sk}+\tau\right)^2} \frac{n - b_{g_k}}{(n-1) \, b_{g_k}} \right] \left(f(\xk) - f(\xopt) \right)\\
&\leq \left(1-\frac{2\mu\etak}{ \left(L_{\Sk}+\tau\right)}+\frac{\mu L\etak^2}{\left(L_{\Sk}+\tau\right)\left(\mu_{\Sk}+\tau\right)}+\frac{\mu L(\rho-1)\etak^2}{\left(\mu_{\Sk}+\tau\right)^2}\frac{n-b_{g_k}}{(n-1)\,b_{g_k}}\right)(f(\xk)-f(\xopt))
\end{align*}
Define constants $c_1=\frac{\mu}{\left(L_{\Sk}+\tau\right)}$, $c_2=\frac{L}{\left(\mu_{\Sk}+\tau\right)}$, and $c_3=\frac{\mu(\rho-1)}{\left(\mu_{\Sk}+\tau\right)}\frac{n-b_{g_k}}{(n-1)\,b_{g_k}}$ using constant batch-sizes for the subsampled gradient and Hessian, \ie $b_{g_k} = b_{g} \geq 1$ and $b_{s_k} = b_{s} \geq 1$,
\begin{align*}
\implies \E_{\Gk}[f(\xkk)]-f(\xopt) &\leq \left(1-2c_1\etak+c_1c_2\etak^2+c_2c_3\etak^2\right)(f(\xk)-f(\xopt)).
\end{align*}
To ensure contraction, the step size needs to satisfy $\left(1-2c_1\etak+c_1c_2\etak^2+c_2c_3\etak^2\right) \in (0, 1]$, which gives
\begin{equation}
    0<\etak\leq\frac{2c_1}{c_2(c_1+c_3)}
\end{equation}
Taking $\etak\leq\frac{c_1}{c_2(c_1+c_3)}$, the above bound becomes
\begin{align*}
\E_{\Gk}[f(\xkk)]-f(\xopt) & \leq \left(1-\frac{c_1^2}{c_2(c_1+c_3)}\right)(f(\xk)-f(\xopt))\\
& \leq \left(1-\underbrace{\frac{\mu \, \left(\mu_{\Sk}+\tau\right)^2}{L \, \left(L_{\Sk}+\tau\right) \, \left(\mu_{\Sk}+\tau\right) + c_g L \, \left(L_{\Sk}+\tau\right)}}_{\defeq\,C}\right)(f(\xk)-f(\xopt)) 
\end{align*}
where we denote $c_g = \frac{(\rho-1)(n - b_g)}{(n-1)\, b_g}$. Note that since $c_g\geq 0$, our bound for $\etak$ simplifies to 
$$\etak \leq \frac{\left(\mu_{\Sk}+\tau\right)^2}{L\left(\left(\mu_{\Sk}+\tau\right)+\left(L_{\Sk}+\tau\right)c_g\right)} \leq \frac{\left(\mu_{\Sk}+\tau\right)^2}{L\left(\mu_{\Sk}+\tau\right)} \leq \frac{\left(\mu_{\Sk}+\tau\right)}{L},$$
hence satisfies the earlier requirement in (\ref{eq:eta_constraint}). Now we lower bound the term C,
\begin{align*}
C & \geq \frac{\mu \, \left(\mu_{\Sk}+\tau\right)^2}{\left(\Ltilde+\tau\right) \, L \, \left(\mu_{\Sk}+\tau + c_g \right)} \geq \frac{\mu \, \left(\mu_{\Sk}+\tau\right)^2}{2 \max\{c_g, \left(\mu_{\Sk}+\tau\right)\} \,\left(\Ltilde+\tau\right) \, L}, & \tag{Since $\frac{a+b}{2} \leq \max\{a,b\}$.} \\
\implies \E_{\Gk}[f(\xkk)]-f(\xopt) & \leq \left(1-
\frac{\mu \, \left(\mu_{\Sk}+\tau\right)^2}{2 \max\{c_g, \left(\mu_{\Sk}+\tau\right)\} \,\left(\Ltilde+\tau\right) \, L} \right)(f(\xk)-f(\xopt)) \\
\intertext{\textbf{Case 1}: If $c_g \geq \left(\mu_{\Sk}+\tau\right)$, then}
 \E_{\Gk}[f(\xkk)]-f(\xopt) & \leq \left(1-
\frac{\left(\mu_{\Sk}+\tau\right)^2 \, \mu}{2 c_g\,\left(\Ltilde+\tau\right) \, L} \right)(f(\xk)-f(\xopt)) .
\intertext{Taking an expectation w.r.t. $\Sk$ yields
}
\E_{\Sk, \Gk}[f(\xkk)]-f(\xopt) & \leq \E_{\Sk} \left(1-
\frac{\left(\mu_{\Sk}+\tau\right)^2 \, \mu}{2 c_g\,\left(\Ltilde+\tau\right) \, L} \right)(f(\xk)-f(\xopt)) \\
& \leq \left(1-
\frac{\left(\E_{\Sk}[\mu_{\Sk}]+\tau\right)^2 \, \mu}{2 c_g\,\left(\Ltilde+\tau\right) \, L} \right)(f(\xk)-f(\xopt))  \tag{By Jensen's inequality} \\
 & \leq \left(1-
\frac{\left(\bar{\mu}+\tau\right)^2 \, \mu}{2 c_g\,\left(\Ltilde+\tau\right) \, L} \right)(f(\xk)-f(\xopt)). \\
\intertext{\textbf{Case 2}: If $c_g \leq (\mu_S + \tau)$, then}
\E_{\Gk}[f(\xkk)]-f(\xopt) & \leq \left(1-
\frac{\mu \, \left(\mu_{\Sk}+\tau\right)}{2 \left(\Ltilde+\tau\right) \, L} \right)(f(\xk)-f(\xopt)) .
\intertext{Taking an expectation w.r.t. $\batchS_k$ yields}
\E_{\Sk, \Gk}[f(\xkk)]-f(\xopt) & \leq  \left(1- \frac{\mu(\bar{\mu}+\tau)}{2 \left(\Ltilde+\tau\right) \, L} \right)(f(\xk)-f(\xopt)) \\
\intertext{Putting the two cases together, we have}
\E_{\Sk,\Gk}[f(\xkk)] - f(\xopt) & \leq
\max \left\{\left(1- \frac{\left(\bar{\mu}+\tau\right)^2 \, \mu}{2 c_g \,\left(\Ltilde+\tau\right) \, L} \right), \left(1- \frac{\mu(\bar{\mu}+\tau)}{2 \left(\Ltilde+\tau\right) \, L} \right) \right\} (f(\xk)-f(\xopt)).
\intertext{Now we take expectation over all time steps and apply recursion,}
\E[f(\x_T)] - f(\xopt) & \leq
\left(\max \left\{\left(1- \frac{\left(\bar{\mu}+\tau\right)^2 \, \mu}{2 c_g \,\left(\Ltilde+\tau\right)  L} \right), \left(1- \frac{\mu(\bar{\mu}+\tau)}{2 \left(\Ltilde+\tau\right)  L} \right) \right\}\right)^{T} (f(\x_0)-f(\xopt)).
\intertext{Simplifying using the definition for the condition number $\kappa=\frac{L}{\mu}$ gives us}
 \E[f(\x_T)] - f(\xopt) & \leq\left(1-\min\left\{\frac{\left(\bar{\mu}+\tau\right)^2}{2\kappa c_g\left(\Ltilde+\tau\right)}\,,\,\frac{(\bar{\mu}+\tau)}{2\kappa\left(\Ltilde+\tau\right)}\right\}\right)^T(f(\x_0)-f(\xopt))
 \intertext{and the proof is complete.}
\end{align*}
\end{proof}

\section{Proof of Theorem~\ref{thm:inexactnewtonlocal}}
\label{app:ssn-local}
\begin{proof}
From the update rule,
\begin{align*}
\norm{\x_{k+1}-\x^*} &= \norm{\x_k-\x^*-\left[\shessf{\batchS_k}{\x_k}\right]^{-1}\sgradf{\batchG_k}{\x_k}}\\
&=\norm{\left[\shessf{\batchS_k}{\x_k}\right]^{-1}\left(\shessf{\batchS_k}{\x_k}(\x_k-\x^*)-\sgradf{\batchG_k}{\x_k}\right)}\\
&=\norm{\left[\shessf{\batchS_k}{\x_k}\right]^{-1}\left(\shessf{\batchS_k}{\x_k}(\x_k-\x^*)-\gradf{\x_k}-\sgradf{\batchG_k}{\x_k}+\gradf{\x_k}\right)} \\
&\leq\norm{\left[\shessf{\batchS_k}{\x_k}\right]^{-1}} \, \norm{\shessf{\batchS_k}{\x_k}(\x_k-\x^*)-\gradf{\x_k}-\sgradf{\batchG_k}{\x_k}+\gradf{\x_k}}\\
&\leq \frac{1}{\left(\mu_{S_k}+\tau\right)} \norm{\shessf{\batchS_k}{\x_k}(\x_k-\x^*)-\gradf{\x_k}-\sgradf{\batchG_k}{\x_k}+\gradf{\x_k}}\\
& = \frac{1}{\left(\mu_{S_k}+\tau\right)} \big\|\shessf{\batchS_k}{\x_k}(\x_k-\x^*) - \hessf{\x_k}(\x_k-\x^*) + \hessf{\x_k}(\x_k-\x^*)\\
&\phantom{aaaaaaaaaaaaaaaaaaaaaaaaaaaaaa}-\gradf{\x_k} - \sgradf{\batchG_k}{\x_k} + \gradf{\x_k}\big\|
\intertext{where we repeatedly applied the triangle inequality. Then we have}
\norm{\x_{k+1}-\x^*} & \leq \frac{1}{\left(\mu_{S_k}+\tau\right)} \Big[ \norm{\hessf{\x_k}(\x_k-\x^*)-\gradf{\x_k}} 
+ \norm{\left(\shessf{\batchS_k}{\x_k}-\hessf{\x_k}\right)(\x_k-\x^*)} \\
& \phantom{aaaaaaaaaaaaaaaaaaaaaaaaaaaaaaaaaaaaaaa}+ \norm{\sgradf{\batchG_k}{\x_k}-\gradf{\x_k}}\Big].
\end{align*}
Taking the expectation $\mathbb{E}_k$ over all combinations of $\batchS_k$ and $\batchG_k$, we have
\begin{align*}
& \mathbb{E}_k \norm{\x_{k+1}-\x^*} \leq \frac{1}{\left(\mutilde+\tau\right)} \Big[ \underbrace{\norm{\hessf{\x_k}(\x_k-\x^*)-\gradf{\x_k}}}_{\text{(1) Bound using Lipschitz Hessian}}  +  \underbrace{\mathbb{E}_k\norm{\left(\shessf{\batchS_k}{\x_k}-\hessf{\x_k}\right)(\x_k-\x^*)}}_{\text{(2) Bound using Hessian Variance}} \\ &\phantom{aaaaaaaaaaaaaaaaaaaaaaaaaaaaaaaaaaaaaaaaaaaaaaaaaa} + \underbrace{\mathbb{E}_k\norm{\sgradf{\batchG_k}{\x_k}-\gradf{\x_k}}}_{\text{(3) Bound using SGC}} \Big].
\end{align*}
We bound the first term using $M$-Lipschitz continuity of the Hessian, 
\begin{align*}
\norm{\hessf{\x_k}(\x_k-\x^*)-\gradf{\x_k}} & = \norm{\grad{\x_k}-\grad{\x^*} - \hessf{\x_k}(\x_k-\x^*)} \\
& = \norm{\int^1_0\hessf{\x^*+t(\x_k-\x^*)}(\x_k-\x^*)dt
- \hessf{\x_k}(\x_k-\x^*)} \\
& = \norm{\int^1_0\left(\hessf{\x^*+t(\x_k-\x^*}-\hessf{\x_k} \right)(\x_k-\x^*) dt} \\
& \leq \int^1_0 \norm{\left(\hessf{\x^*+t(\x_k-\x^*}-\hessf{\x_k} \right)(\x_k-\x^*) dt} \tag{Jensen's inequality}\\
& \leq  \int_0^1 \norm{\hessf{\x^*+t(\x_k-\x^*)}-\hessf{\x_k}} \, \norm{\x_k-\x^*} dt \tag{Cauchy–Schwarz inequality} \\
& \leq \int_0^1 M \norm{\x^*+t(\x_k-\x^*) - \x_k} \, \norm{\x_k-\x^*} dt \tag{$M$-Lipschitz Hessian}\\
&=\normsq{\x_k-\x^*} \int_0^1 M \, (1-t) \, dt \nonumber\\
\implies \norm{\hessf{\x_k}(\x_k-\x^*)-\gradf{\x_k}} & \leq \frac{M}{2} \normsq{\x_k-\x^*} 
\end{align*}.

To bound the second term, we use
\begin{align*}
\E_k\norm{\shessf{\batchS_k}{\xk}-\hessf{\xk}} &= \E_k\norm{\nabla^2f_{\batchS_k}(\xk)+\tau I -\hessf{\xk}} \\
&\leq \E_k \norm{\nabla^2f_{\batchS_k}(\xk)-\hessf{\xk}} + \norm{\tau I}.
\intertext{By the assumption that the subsampled Hessians have bounded variance, from \citet{harikandeh2015stopwasting, lohr2019sampling} we have that}
\E_k\left[\normsq{\nabla^2f_{\batchS_k}(\xk)-\hessf{\xk}}\right] &\leq \frac{n - b_{s_k}}{n \, b_{s_k}}\sigma_s^2.
\intertext{Using Jensen's inequality for the square root function and combining with the above, we have}
\E_k\norm{\shessf{\batchS_k}{\xk}-\hessf{\xk}} &\leq \sigma_s\sqrt{\frac{n - b_{s_k}}{n \, b_{s_k}}} + \tau.
\end{align*}

Then by setting the sub-sampled Hessian batch-size according to
\begin{align*}
b_{s_k} \geq\frac{n}{\norm{\gradf{\x_k}}\frac{n}{\sigma_s^2}+1},
\end{align*}
we can achieve 
\begin{align*}
\E_k \norm{\shessf{\batchS_k}{\x_k}-\hessf{\x_k}}\leq \norm{\gradf{\x_k}} + \tau.
\end{align*}
The second term can then be bounded as
\begin{align*}
\E_k \norm{\left(\shessf{\batchS_k}{\x_k}-\hessf{\x_k}\right)(\x_k-\x^*)} 
& \leq \norm{\x_k-\x^*} \E_k \norm{\shessf{\batchS_k}{\x_k}-\hessf{\x_k}} \tag{Cauchy-Schwarz}\\
&\leq\norm{\x_k-\x^*} \, \left(\norm{\gradf{\x_k}}+\tau\right) \tag{from above}\\
& \leq L\normsq{\x_k-\x^*} + \tau \norm{\x_k - \x^*} \tag{$L$-smoothness}.
\end{align*}
The third term can again be bounded using Lemma~\ref{lemma:grad-batch-sgc}. Applying Jensen's inequality, this gives us
\begin{align*}
\E_k \norm{\sgradf{\batchG_k}{\x_k}-\gradf{\x_k}} \leq \sqrt{\frac{n - b_{g_k}}{(n-1) \, b_{g_k}}}\sqrt{\rho-1}\norm{\gradf{\x_k}}.
\end{align*}
If we let $b_{g_k} \geq \frac{n}{\left(\frac{n-1}{\rho-1}\right)\normsq{\gradf{\x_k}}+1}
$ then we have 
\begin{align*}
\E_k \norm{\sgradf{\batchG_k}{\x_k}-\gradf{\x_k}} \leq \normsq{\gradf{\x_k}} \leq L^2\normsq{\x_k-\x^*},
\end{align*}
where the last inequality again comes from $L$-smoothness. Putting the above three bounds together gives us
\begin{align*}
\mathbb{E}_k\norm{\x_{k+1}-\x^*} &\leq \frac{1}{(\mutilde+\tau)} \, \left[ \frac{M}{2} \normsq{\x_k-\x^*} + L \normsq{\x_k-\x^*} + \tau\norm{\xk-\xopt} + L^2 \normsq{\x_k-\x^*} \right]\\
& \leq \frac{\left(M + 2 L + 2 L^2 \right)}{2 (\mutilde+\tau)} \, \normsq{\x_k-\x^*} + \frac{\tau}{\mutilde+\tau}\norm{\xk-\xopt}.
\end{align*}

\begin{align*}
\intertext{Using the $\gamma$ bounded moments assumption by taking expectation over all $k$, this implies}
\mathbb{E} \norm{\x_{k+1}-\x^*} &\leq \frac{\gamma \left(M + 2 L + 2 L^2 \right)}{2 (\mutilde+\tau)} \, \left(\E \norm{\x_k-\x^*} \right)^2 + \frac{\tau}{\mutilde+\tau}\E\norm{\xk-\xopt},
\end{align*}
which gives us the linear-quadratic convergence. Moreover, if $\tau=0$ and $\mutilde>0$, we have the following quadratic convergence 
\begin{equation*}
\E\norm{\xkk-\xopt} \leq \frac{\gamma(M+2L+2L^2)}{2\mutilde}\,\left(\E\norm{\xk-\xopt}\right)^2,
\end{equation*}
given that $\norm{\x_0-\x^*}\leq \frac{2\mutilde}{\gamma(M+2L+2L^2)}$. Here, convergence is guaranteed as $\E\norm{\xk-\xopt}\leq \frac{2\mutilde}{\gamma(M+2L+2L^2)}$ for all $k$ by induction using the neighbourhood criterion. 
\end{proof}

\subsection{Proof of Corollary~\ref{corollary:decreasing-tau}}
\label{sec:decreasing-tau}
\begin{proof}
By a similar analysis of Theorem~\ref{thm:inexactnewtonlocal} but replacing $\tau$ with $\tau_k$ and using $\mutilde_k$ instead of $\mutilde$, we arrive at
\begin{align*}
\mathbb{E}_k \norm{\x_{k+1}-\x^*} &\leq \frac{1}{\left(\mutilde_k+\tau_k\right)} \Big[ \underbrace{\norm{\hessf{\x_k}(\x_k-\x^*)-\gradf{\x_k}}}_{\text{(1) Bound using Lipschitz Hessian}}  +  \underbrace{\mathbb{E}_k\norm{\left(\shessf{\batchS_k}{\x_k}-\hessf{\x_k}\right)(\x_k-\x^*)}}_{\text{(2) Bound using Hessian Variance}} \\ &\phantom{aaaaaaaaaaaaaaaaaaaaaaaaaaaaaaaaaaaaaa} + \underbrace{\mathbb{E}_k\norm{\sgradf{\batchG_k}{\x_k}-\gradf{\x_k}}}_{\text{(3) Bound using SGC}} \Big].
\end{align*}
The second term can then be bounded as
\begin{align*}
\E_k \norm{\left(\shessf{\batchS_k}{\x_k}-\hessf{\x_k}\right)(\x_k-\x^*)} 
& \leq \norm{\x_k-\x^*} \E_k \norm{\shessf{\batchS_k}{\x_k}-\hessf{\x_k}}\\
&\leq\norm{\x_k-\x^*} \, \left[\norm{\gradf{\x_k}}+\tau_k\right]\\
& \leq L\normsq{\x_k-\x^*} + \tau_k \norm{\x_k - \x^*},
\intertext{and if we decrease the regularization factor as $\tau_k\leq \norm{\gradf{\xk}}$,}
& \leq L\normsq{\x_k-\x^*} + \norm{\gradf{\xk}}\norm{\xk-\xopt}\\
&\leq 2L\normsq{\xk-\xopt}.
\end{align*}

The other two terms will be bounded similarly as in Theorem~\ref{thm:inexactnewtonlocal}. Putting all three bounds together,
\begin{align*}
\mathbb{E}_k\norm{\x_{k+1}-\x^*} & \leq \frac{1}{(\mutilde_k+\tau_k)} \, \left[ \frac{M}{2} \normsq{\x_k-\x^*} + 2L \normsq{\x_k-\x^*}  + L^2 \normsq{\x_k-\x^*} \right]\\
& \leq \frac{\left(M + 4L + 2 L^2 \right)}{2 (\mutilde_k+\tau_k)} \, \normsq{\x_k-\x^*} .
\end{align*}

\begin{align*}
\intertext{Using the $\gamma$ bounded moments assumption,}
\implies \mathbb{E} \norm{\x_{k+1}-\x^*} &\leq \frac{\gamma \left(M + 4L + 2 L^2 \right)}{2 (\mutilde_k+\tau_k)} \, \left(\E \norm{\x_k-\x^*} \right)^2,
\end{align*}
which will converge in a local neighbourhood $\norm{\xk-\xopt}\leq\frac{2(\mutilde+\min_k\tau_k)}{\gamma(M+4L+2L^2)}$.
\end{proof}

\subsection{Local quadratic convergence under stronger SGC}
\label{app:ssn-local-stronger}
\begin{corollary}
Under the same assumptions (a) - (c) of Theorem~\ref{thm:globalinexactnewton}, along with (d) $\rho$-stronger SGC, (e) $M$-Lipschitz continuity of the Hessian, and (f) $\gamma$-bounded moment of the iterates, the sequence $\{\xk\}_{k\geq 0}$ generated by R-SSN with (\rnum{1}) unit step-size $\eta_k = \eta =1$ and (\rnum{2}) constant batch-size $b_{g_k} = b_g$ for the gradients with a growing batch-size for Hessian such that 
\begin{align*}
\qquad b_{s_k}\geq\frac{n}{\frac{n}{\sigma^2}\norm{\gradf{\x_k}}+1}
\end{align*}
converges to $\xopt$ with the quadratic rate, 
\begin{align*}
\E_k\norm{\x_{k+1}-\x^*}\leq\gamma\left(\frac{M+2L+2L^2 \, c_g}{2\mutilde}\right)\left(\E\norm{\x_k-\x^*}\right)^2,
\end{align*}
from a close enough initialization $\x_0$ such that $\norm{\x_0-\x^*}\leq\frac{2\mutilde}{\gamma \, (M+2L+2L^2 \, c_g)}$. Here $c_g =\sqrt{\frac{\rho(n - b_g)}{b_g \, (n-1)}}$. 
\label{cor:SSN-local-stronger-SGC}
\end{corollary}
\begin{proof}
Using Lemma~\ref{lemma:grad-batch-strong-sgc} to bound the third term in the analysis of Theorem~\ref{thm:inexactnewtonlocal}, we obtain  
\begin{align*}
\E_{k} \norm{\sgradf{\batchG_k}{\x_k}-\gradf{\x_k}} & \leq \sqrt{\frac{\rho \, (n - b_{g_k})}{(n-1) \, b_{g_k}}}\normsq{\gradf{\x_k}} \\
&\leq L^2\sqrt{\frac{\rho \, (n - b_{g_k})}{(n-1) \, b_{g_k}}}\normsq{\x_k-\x^*}.
\end{align*}
The first two terms are bounded the same way, giving us
\begin{align*}
\mathbb{E}_k\norm{\x_{k+1}-\x^*} &\leq \frac{1}{(\mutilde+\tau)} \, \left[ \frac{M}{2} \normsq{\x_k-\x^*} + L \normsq{\x_k-\x^*} + \tau\norm{\xk-\xopt} + L^2 c_g \, \normsq{\x_k-\x^*} \right] \\
&\phantom{aaaaaaaaaaaaaaaaaaaaaaaaaaaaaaaaaaaaaaaaaaaaaa}.
\end{align*}
where $c_g = \sqrt{\frac{\rho \, (n - b_{g})}{(n-1) \, b_{g}}}$ and $b_{g}$ can remain fixed for all the iterations. 

\begin{align*}
\intertext{Using the $\gamma$ bounded moments assumption,}
\implies \mathbb{E} \norm{\x_{k+1}-\x^*} \leq \frac{\gamma \left(M + 2 L + 2 L^2 c_g \right)}{2 (\mutilde+\tau)} \, \left(\E \norm{\x_k-\x^*} \right)^2 + \frac{\tau}{\mutilde+\tau}\E\norm{\xk-\xopt}
\end{align*}
which gives us the linear-quadratic convergence. 
Furthermore if $\tau = 0$, $\mutilde > 0$ and $\norm{\x_0-\x^*}\leq \frac{2(\mutilde + \tau)}{\gamma \, (M+2L+2L^2 \, c_g)}$ yields the quadratic rate
\begin{equation*}
\E \norm{\x_{k+1}-\x^*} \leq \left(\frac{\gamma \left(M+2L+2L^2 \, c_g \right)}{2(\mutilde + \tau)}\right)\left(\E\norm{\x_k-\x^*}\right)^2.
\end{equation*}
\end{proof}

\section{Proof of Theorem~\ref{thm:full-stochastic-self-concord}}
\label{app:full-stochastic-self-concord}
We define the local norm of a direction $h$ with respect to the local Hessian for a self-concordant function as
\begin{align*}
\norm{h}_x = \inner{\hessf{x}h}{h}^{1/2}=\norm{\left[\hessf{x}\right]^{1/2}h}.
\end{align*}
The following are standard results for self-concordant functions needed in our analysis.

\begin{theorem}[Theorem 5.1.9 in \cite{nesterov2018lectures}]
\label{thm:nesterov-self-concordant-descent}
Let $f$ be a standard self-concordant function, $x,y\in\dom f$, and $\norm{y-x}_x<1$. Then
\begin{equation*}
    f(y)\leq f(x)+\inner{\gradf{x}}{y-x}+\omegastarof{\norm{y-x}_x}.
\end{equation*}
Here, $\omegastarof{t} = -t-\log(1-t)$.
\end{theorem}

\begin{theorem}[Theorem 5.2.1 in \cite{nesterov2018lectures}]
\label{thm:nesterov-self-concordant-strong-convexity}
Let $\lambda^0(\x)\defeq \inner{\gradf{\x}}{\left[\hessf{\x}\right]^{-1}\gradf{\x}}^{1/2}$ be the unregularized Newton decrement at $\x$. If $f$ is standard self-concordant and $\lambda^0(\x) < 1$, then
\begin{equation*}
f(\x)-f(\xopt)\leq\omegastarof{\lambda^0(\x)}.
\end{equation*}
\end{theorem}

See \cite{nesterov2018lectures} for the proofs. Now we give the proof for Theorem~\ref{thm:full-stochastic-self-concord}.
\begin{proof}
First, we analyze the norm of the update direction with respect to the local Hessian,
\begin{align}
\label{eq:normupdate}
\norm{\xkk-\xk}_{\xk} &=\frac{c\eta}{1+\eta\tilde{\lambda}_k}\norm{\left[\hessf{\xk}\right]^{1/2}\left[\shessf{\batchS_k}{\xk}\right]^{-1}\sgradf{\batchG_k}{\xk}} \nonumber\\
&= \frac{c\eta}{1+\eta\tilde{\lambda}_k}\norm{\left[\hessf{\xk}\right]^{1/2}\left[\shessf{\batchS_k}{\xk}\right]^{-1/2}\left[\shessf{\batchS_k}{\xk}\right]^{-1/2}\sgradf{\batchG_k}{\xk}} \nonumber\\
&\leq \frac{c\eta}{1+\eta\tilde{\lambda}_k}\norm{\left[\hessf{\xk}\right]^{1/2}\left[\shessf{\batchS_k}{\xk}\right]^{-1/2}}\norm{\left[\shessf{\batchS_k}{\xk}\right]^{-1/2}\sgradf{\batchG_k}{\xk}} \nonumber \\
&\leq \frac{c\eta\tilde{\lambda}_k}{1+\eta\tilde{\lambda}_k} \sqrt{\frac{L}{\mutilde+\tau}},  \nonumber \\
\intertext{where the last inequality comes from our definition of the regularized stochastic Newton decrement and regularized subsampled Hessian. Substitute in the choice of $c$,}
\implies \norm{\xkk-\xk}_{\xk} &\leq \frac{\eta\tilde{\lambda}_k}{1+\eta\tilde{\lambda}_k} \leq 1 .
\end{align}
This allows us to analyze the sub-optimality in terms of objective values using Theorem~\ref{thm:nesterov-self-concordant-descent},
\begin{align}
\label{eq:expwrtboth}
f(\xkk) &\leq f(\xk) -\frac{c\eta}{1+\eta\tilde{\lambda}_k}\inner{\gradf{\xk}}{\left[\shessf{\batchS_k}{\xk}\right]^{-1}\sgradf{\batchG_k}{\xk}} + \omegastarof{\norm{\xkk-\xk}_{\xk}}. \nonumber \\
\intertext{We know that $\omega_*$ is strictly increasing on the positive domain, using (\ref{eq:normupdate}) we have,}
\implies f(\xkk) &\leq f(\xk) -\frac{c\eta}{1+\eta\tilde{\lambda}_k}\inner{\gradf{\xk}}{\left[\shessf{\batchS_k}{\xk}\right]^{-1}\sgradf{\batchG_k}{\xk}} + \omegastarof{\frac{\eta\tilde{\lambda}_k}{1+\eta\tilde{\lambda}_k}} \nonumber\\
& \leq f(\xk) -\frac{c\eta}{1+\eta\tilde{\lambda}_k}\inner{\gradf{\xk}}{\left[\shessf{\batchS_k}{\xk}\right]^{-1}\sgradf{\batchG_k}{\xk}} + \omegastarof{\omegaprimeof{\eta\tilde{\lambda}_k}} \nonumber\\
\intertext{since $\omega'(t)=\frac{t}{1+t}$. Now take expectation on both sides with respect to $\batchG_k$ and $\batchS_k$  conditioned on $\xk$,}
\expectationwrtof{\batchG_k,\batchS_k}{f(\xkk)} &\leq f(\xk) - \expectationwrtof{\batchG_k,\batchS_k}{\frac{c\eta}{1+\eta\tilde{\lambda}_k}\inner{\gradf{\xk}}{\left[\shessf{\batchS_k}{\xk}\right]^{-1}\sgradf{\batchG_k}{\xk}}} + \expectationwrtof{\batchG_k, \batchS_k}{\omegastarof{\omegaprimeof{\eta\tilde{\lambda}_k}}} .
\end{align}
To lower bound the middle term, observe that 
\begin{align*}
\tilde{\lambda}_k &\leq \frac{1}{\sqrt{\mutilde+\tau}}\norm{\sgradf{\batchG_k}{\xk}}\\
&=\frac{1}{\sqrt{\mutilde+\tau}}\norm{\sgradf{\batchG_k}{\xk}-\sgradf{\batchG}{\xopt}} \tag{by interpolation assumption}\\
&\leq  \frac{\Ltilde+\tau }{\sqrt{\mutilde+\tau}}\norm{\xk-\xopt} \tag{by smoothness on the batch}\\
&\leq \frac{(\Ltilde+\tau) D}{\sqrt{\mutilde+\tau}} \\
&\defeq \lambda_{\max},
\end{align*}
\begin{align*}
 \implies \expectationwrtof{\batchG_k,\batchS_k}{f(\xkk)}  &\leq f(\xk) - \frac{c\eta\lambda_k^2}{1+\eta\lambda_{\max}} + \expectationwrtof{\batchG_k, \batchS_k}{\omegastarof{\omegaprimeof{\eta\tilde{\lambda}_k}}} .
\end{align*}
Using $\omegastarof{\omegaprimeof{t}}=t\omegaprimeof{t}-\omegaof{t}$ for $t\geq0$, the last term can be bounded as
\begin{align}
 \expectationof{\omegastarof{\omegaprimeof{\eta\tilde{\lambda}_k}}} &= \expectationof{\eta\tilde{\lambda}_k\omegaprimeof{\eta\tilde{\lambda}_k}-\omegaof{\eta\tilde{\lambda}_k}} \nonumber \\
 &\leq \frac{\eta^2\expectationof{\tilde{\lambda}_k^2}}{1+\eta\tilde{\lambda}_{\min}} - \expectationof{\omegaof{\eta\tilde{\lambda}_k}} \nonumber\\
 \intertext{where $\tilde{\lambda}_{\min}=\min_{\xk,\batchG_k,\batchS_k} \tilde{\lambda}_k$. Applying the Newton decrement SGC gives us }
 \implies \expectationof{\omegastarof{\omegaprimeof{\eta\tilde{\lambda}_k}}} &\leq \frac{\eta^2\rhond\lambda_k^2}{1+\eta\tilde{\lambda}_{\min}}-\expectationof{\omegaof{\eta\tilde{\lambda}_k}} \nonumber
\end{align}
with $\rhond=\frac{\rho L}{\mutilde+\tau}$. Combining this with (\ref{eq:expwrtboth}) gives us
\begin{align}
\label{eq:expwrtboth2}
\expectationof{f(\xkk)} &\leq f(\xk) - \frac{c\eta\lambda_k^2}{1+\eta\lambda_{\max}} + \frac{\eta^2\rhond\lambda_k^2}{1+\eta\tilde{\lambda}_{\min}}-\expectationof{\omegaof{\eta\tilde{\lambda}_k}} \nonumber\\
 &= f(\xk) - \eta\lambda_k^2\underbrace{\left(\frac{c}{1+\eta\lambda_{\max}}-\frac{\eta\rhond}{1+\eta\tilde{\lambda}_{\min}}\right)}_{(*)}-\expectationof{\omegaof{\eta\tilde{\lambda}_k}} .
\end{align}
For $\eta$ in the range $0< \eta \leq\frac{c}{\rhond(1+\lambda_{\max})-c\lambda_{\min}} \quad (\leq 1)$, we have
\begin{align*}
\eta\rhond(1+\lambda_{\max}) - \eta c\lambda_{\min} &\leq c\\
\implies \eta\rhond(1+\eta\lambda_{\max}) \leq \eta\rhond(1+\lambda_{\max}) &\leq c(1+\eta\lambda_{\min})\\
\implies \frac{\eta\rhond}{1+\eta\lambda_{\min}} &\leq \frac{c}{1+\eta\lambda_{\max}}.
\end{align*}
The $\eta \leq 1$ claim comes from substituting in our choice of $c$ and definition of $\rhond$, which gives us
\begin{align*}
\eta &\leq     \frac{c}{\rhond(1+\lambda_{\max})-c\lambda_{\min}} \\
&=\frac{1}{\frac{\rho(1+\lambda_{\max})}{\left(\frac{\mutilde+\tau}{L}\right)^{3/2}}-\lambda_{\min}},
\end{align*}
which is less than $1$ since $\rho>1$ and $\lambda_{\max}>\lambda_{\min}$ for $\tau$ chosen small enough. Thus we can choose $0<\eta\leq\frac{c}{\rhond(1+\lambda_{\max})}\,(\leq \frac{c}{\rhond(1+\lambda_{\max})-c\lambda_{\min}})$ and upper bound (*) in (\ref{eq:expwrtboth2}) by $0$. Now we have the following expected decrease of the function value for one update,
\begin{align}
\label{eq:expdecrboth}
\expectationof{f(\xkk)} &\leq f(\xk) - \expectationof{\omegaof{\eta\tilde{\lambda}_k}} \nonumber\\
\intertext{by the convexity of $\omega$ and using Jensen's inequality,}
&\leq f(\xk) - \omegaof{\eta\expectationof{\tilde{\lambda}_k}} \nonumber\\
&= f(\xk) - \omegaof{\eta\E\norm{\shessf{\batchS_k}{\xk}^{-1/2}\sgradf{\batchG_k}{\xk}}} \nonumber\\
\intertext{apply Jensen's inequality using the convexity of $\norm{\cdot}_H$ for some $H\succeq0$ and that $\omega$ is monotonically increasing on the positive domain and the fact that $\Gk$ and $\Sk$ are independent batches,}
&\leq f(\xk) - \omegaof{\eta\norm{\E\,\shessf{\batchS_k}{\xk}^{-1/2}\E\,\sgradf{\batchG_k}{\xk}}}. \nonumber\\
\intertext{Using the fact that the inverse square root function is operator convex (L{\"o}wner-Heinz Theorem) \citep{tropp2015introduction,carlen2010trace} on a positive spectrum, we can apply the operator Jensen inequality to bound the inner term,}
\implies \mathbb{E}[f(\xkk)]&\leq f(\xk) - \omegaof{\eta\norm{\left[\hessf{\xk} + \tau I_d \right]^{-1/2}\gradf{\xk}}} \nonumber\\
&= f(\xk) - \omegaof{\eta\lambda_k}. 
\end{align}
 Note that for any $c,\delta\in(0,1]$ and $t\geq 0$,
\begin{align*}
\omegaof{ct}-c\delta\omegaof{t} &= ct-\log(1+ct) - c\delta t+c\delta\log(1+t)\\
&\geq ct-c\log(1+t) - c\delta t+c\delta\log(1+t)\\
&= (c-c\delta)t- (c-c\delta)\log(1+t)\\
&=  (c-c\delta)(t-\log(1+t))\\
&=  (c-c\delta)\omegaof{t}\\
&\geq 0\\
\implies \omegaof{ct} &\geq c\delta\omegaof{t}.
\end{align*}
Combining this with ($\ref{eq:expdecrboth}$) yields the global R-linear convergence rate, 
\begin{align*}
\expectationof{f(\xkk)} &\leq f(\xk) - \eta\delta\omegaof{\lambda_k} \\
\implies \expectationof{f(\x_T)}-f(\xopt) &\leq f(\x_0)-f(\xopt)-\eta\delta\left(\sum_{k=0}^{T-1}\omegaof{\lambda_k}\right).
\end{align*}
Note that $\lambda_k=\inner{\gradf{\xk}}{\left[\hessf{\xk}+\tau I\right]^{-1}\gradf{\xk}}^{1/2} \leq \lambda_k^0$, and since $\omegastarof{t}$ is a decreasing function for $t\leq1/6$, then $\omegastarof{\lambda_k}\geq\omegastarof{\lambda_k^0}$. As shown in \cite{zhang2015disco}, for all $t\leq 1/6$, $\omegastarof{t}\leq 1.26\,\omegaof{t}$, then for $\lambda_k\leq\lambda_k^0\leq 1/6$, we can bound the above as
\begin{align*}
\expectationof{f(\xkk)} &\leq f(\xk) - \frac{\eta\delta}{1.26}\omegastarof{\lambda_k^0}.\\
\intertext{Subtract $f(\xopt)$ from both sides,}
&\leq f(\xk) - f(\xopt) - \frac{\eta\delta}{1.26}\omegastarof{\lambda_k^0}\\
\intertext{and apply Theorem~\ref{thm:nesterov-self-concordant-strong-convexity},}
&\leq f(\xk) - f(\xopt) - \frac{\eta\delta}{1.26}(f(\xk)-f(\xopt))\\
\expectationof{f(\xkk)} - f(\xopt)&\leq \left(1-\frac{\eta\delta}{1.26}\right)(f(\xk)-f(\xopt))
\end{align*}
which completes the proof.
\end{proof}
\section{Proof of Theorem~\ref{thm:lbfgs}}
\label{sec:proof-lbfgs}
\begin{proof}
From the L-smoothness assumption, we have
\begin{align*}
\E_{\Gk}[f(\xkk)] &\leq f(\xk)-\etak\inner{\gradf{\xk}}{\Hk \, \E_{\Gk} [\sgradf{\batchG_k}{\xk}]} + \frac{L}{2}\etak^2 \, \E_{\Gk}\normsq{\Hk \, \sgradf{\batchG_k}{\xk}}\\
&\leq f(\xk)-\etak\inner{\gradf{\xk}}{\Hk\gradf{\xk}} + \frac{L \, \lambda_d^2 \, \etak^2}{2} \, \E_{\Gk}\normsq{\sgradf{\batchG_k}{\xk}}.\\
\intertext{Bounding the last term using Lemma~\ref{lemma:grad-batch-sgc},}
\E_{\Gk}\normsq{\sgradf{\batchG_k}{\xk}} & \leq \left(\frac{(n - b_g) \, (\rho - 1)}{(n-1) \, b_g} + 1 \right) \, \normsq{\gradf{\xk}} .\\
\intertext{Denoting $\left(\frac{(n - b_g) \, (\rho - 1)}{(n-1) \, b_g} + 1 \right)$ as $\rho^{\prime}$,  the expected decrease becomes}
\E_{\Gk}[f(\xkk)] & \leq  f(\xk)-\etak\lambda_1\normsq{\gradf{\xk}}+\frac{\rho^{\prime} L \, \lambda_d^2 \, \etak^2}{2}\normsq{\gradf{\xk}}.\\
\intertext{Let $\etak = \eta = \frac{\lambda_1}{\rho^{\prime} L\lambda_d^2}$,}
\implies \E_{\Gk}[f(\xkk)] & \leq f(\xk)-\frac{\lambda_1^2}{\rho^{\prime} L\lambda_d^2}\normsq{\gradf{\xk}} + \frac{\rho^{\prime} L\lambda_d^2}{2}\frac{\lambda_1^2}{\rho^{{\prime}^2} L^2\lambda_d^4}\normsq{\gradf{\xk}}\\
&= f(\xk)-\left(\frac{\lambda_1^2}{\rho^{\prime} L\lambda_d^2}-\frac{\lambda_1^2}{2\rho^{\prime} L\lambda_d^2}\right)\normsq{\gradf{\xk}}\\
& = f(\xk)-\frac{\lambda_1^2}{2\rho^{\prime} L\lambda_d^2}\normsq{\gradf{\xk}}.\\
\intertext{Subtracting $f(\xopt)$ from both sides and apply strong convexity,}
\E_{\Gk}[f(\xkk)]-f(\xopt) &\leq f(\xk)-f(\xopt)-\frac{\mu\lambda_1^2}{\rho^{\prime} L\lambda_d^2}(f(\xk)-f(\xopt))\\
&= \left(1-\frac{\mu\lambda_1^2}{\rho^{\prime} L\lambda_d^2}\right)(f(\xk)-f(\xopt)) \\
& \leq 
\left(1-\frac{\mu \, \lambda_1^2}{\left(\frac{(n - b_g) \, (\rho - 1)}{(n-1) \, b_g} + 1 \right) \, L \, \lambda_d^2}\right)(f(\xk)-f(\xopt)).
\intertext{After applying recursion gives us the desired result, }
\E[f(\x_{T})]-f(\xopt) & \leq 
\left(1-\frac{\mu \, \lambda_1^2}{\left(\frac{(n - b_g) \, (\rho - 1)}{(n-1) \, b_g} + 1 \right) \, L \, \lambda_d^2}\right)^{T} (f(\x_0)-f(\xopt)).
\end{align*}
\end{proof}
\section{Common Results}
\label{app:common-lemmas}
\begin{lemma}
Consider $y=\frac{1}{n}\sum_{i=1}^ny_i$ where $y_i\in\R^d$. Then for a $y_i$ selected uniformly at random, we have $\expectationof{y_i}=y$. Suppose we uniformly draw a sample $B\subset\{1,\dots,n\}$ and let $y_B=\frac{1}{b}\sum_{i\in B}$ where $b=|B|$. If the $y_i$'s satisfy a growth condition such that
\begin{equation*}
    \E_i\normsq{y_i}\leq c\normsq{y}
\end{equation*}
for some $c>0$. Then the expected squared norm of the error $\epsilon=y_B-y$ can be bounded as
\begin{equation*}
    \E\normsq{y_B-y} \leq \frac{(n-b)(c-1)}{(n-1)b}\normsq{y}.
\end{equation*}
\label{lemma:grad-batch-sgc}
\end{lemma}
\begin{proof}
For an arbitrary entry $j\in\{1,\dots,d\}$, the error $\epsilon_j^2$ can be bounded using its sample variance \citep{lohr2019sampling} as follows,
\begin{align}
\label{eq:expected-size-of-error}
\epsilon_j^2 &= \frac{n-b}{nb}\frac{1}{n-1}\sum_{i=1}^n((y_{i})_j-y_j)^2 \nonumber\\
&= \frac{n-b}{nb}\frac{1}{n-1}\sum_{i=1}^n((y_{i})_j^2-2(y_{i})_jy_j+y_j^2). \nonumber
\intertext{Take the squared norm of $\epsilon$, we have}
\normsq{\epsilon}&=\frac{n-b}{nb}\frac{1}{n-1}\sum_{j=1}^d\sum_{i=1}^n((y_{i})_j^2-2(y_{i})_jy_j+y_j^2)  \nonumber\\
&=\frac{n-b}{nb}\frac{1}{n-1}\sum_{i=1}^n\left(\normsq{y_i}-2\inner{y_i}{y}+\normsq{y}\right). \nonumber
\intertext{Now take expectation on both sides and use the unbiasedness of $y_i$ gives us}
\E\normsq{\epsilon}&=\frac{n-b}{nb}\frac{1}{n-1}\sum_{i=1}^n\left(\E\normsq{y_i}-2\normsq{y}+\normsq{y}\right). 
\intertext{Apply the growth condition,}
\E\normsq{\epsilon} &\leq \frac{n-b}{nb}\frac{1}{n-1} \sum_{i=1}^n\left(c\normsq{y}-\normsq{y}\right) \nonumber\\
&=\frac{(n-b)(c-1)}{(n-1)b}\normsq{y}\nonumber
\end{align} 
which completes the proof.
\end{proof}

\begin{lemma}
Consider the same setup as in Lemma~\ref{lemma:grad-batch-sgc}. If we replace the growth condition with 
\begin{equation*}
    \E_i\normsq{y_i}\leq c\norm{y}^4,
\end{equation*}
then we obtain the following bound on the expected squared error
\begin{equation*}
    \E\normsq{y_B-y} \leq \frac{(n-b)c}{(n-1)b}\norm{y}^4
\end{equation*}
for some $c>0$.
\label{lemma:grad-batch-strong-sgc}
\end{lemma}
\begin{proof}
Using the same analysis as in Lemma~\ref{lemma:grad-batch-sgc} up to equation (\ref{eq:expected-size-of-error}) and applying the new growth condition gives us
\begin{align*}
\E\normsq{\epsilon} &\leq \frac{n-b}{nb}\frac{1}{n-1}\sum_{i=1}^n\left(c\norm{y}^4-\normsq{y}\right)\\
&\leq\frac{n-b}{nb}\frac{1}{n-1}\sum_{i=1}^n\left(c\norm{y}^4\right) \tag{Since $\normsq{y}>0$}\\
&=\frac{(n-b)c}{(n-1)b}\norm{y}^4.
\end{align*}
\end{proof}

\begin{proposition}
\label{prop:newton-decr-sgc}
Suppose function $f$ satisfies the SGC with parameter $\rho$. For any positive definite matrices $A, B$ with bounded eigenvalues $\lambda_{\min}(B)$ and $\lambda_{\max}(A)$, the following inequality holds:
\begin{equation*}
\E\normsq{\sgradf{\batchG}{\x}}_A \leq \frac{\rho\lambda_{\max}(A)}{\lambda_{\min}(B)}\normsq{\gradf{\x}}_B.
\end{equation*}.
\end{proposition}

\begin{proof}
From the LHS, we have
\begin{align*}
\E\normsq{\sgradf{\batchG}{\x}}_A &\leq \lambda_{\max}(A)\E\normsq{\sgradf{\batchG}{\x}} \\
&\leq \rho\lambda_{\max}(A)\normsq{\gradf{\x}}.
\end{align*}
From the RHS, we have
\begin{align*}
\normsq{\gradf{\x}}_B &\geq \lambda_{\min}(B)\normsq{\gradf{\x}}.
\end{align*}
Combining the two inequalities gives us the desired result.
\end{proof}

\newpage
\section{Additional Experiments}
\label{app:additional-exps}

\begin{figure*}[h]
\centering
\includegraphics[scale=0.4]{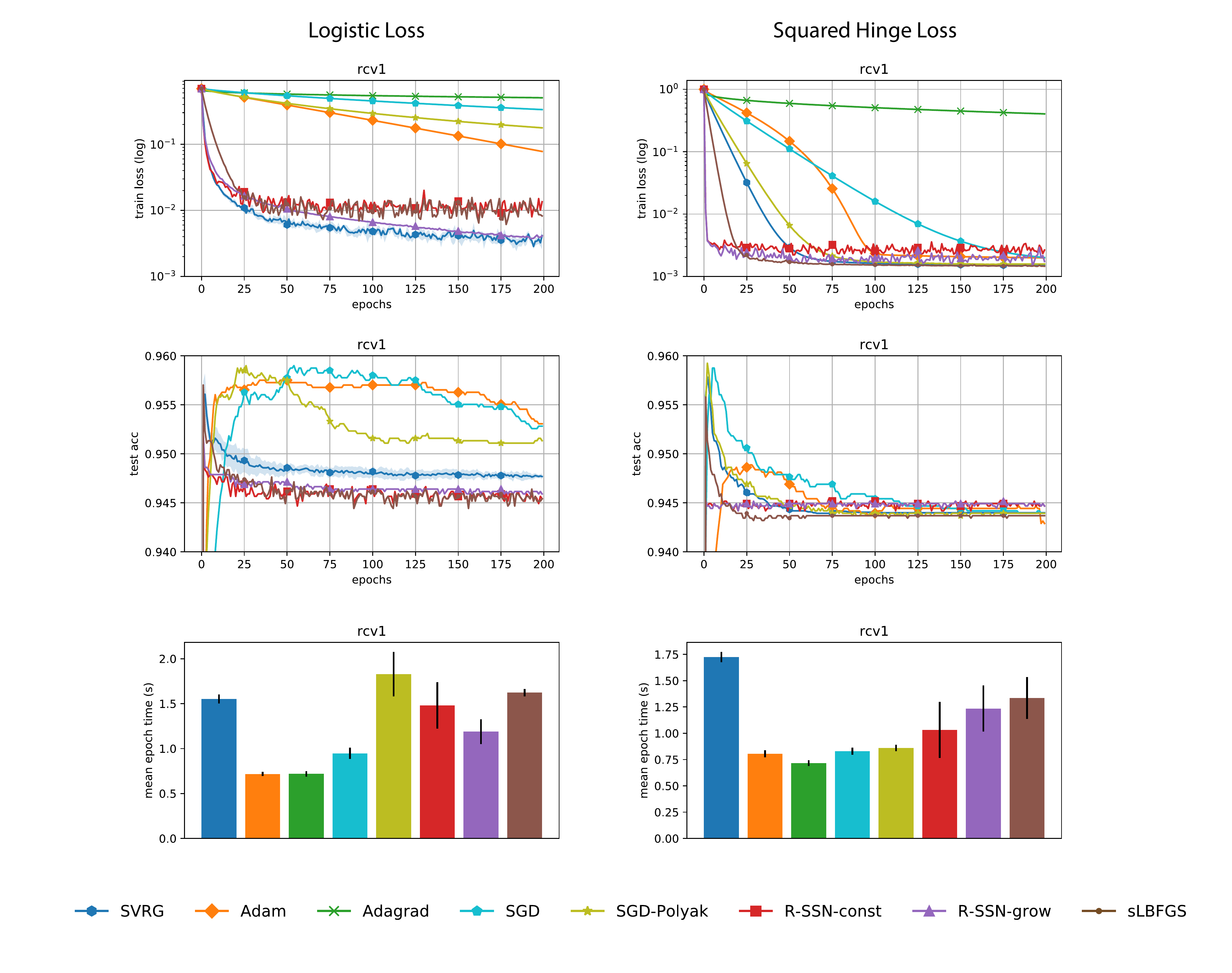}
\caption{Comparison of R-SSN variants and stochastic L-BFGS against first order methods on \texttt{rcv1} for linear models under two different losses. }
\label{fig:exp-kernel-supp}
\end{figure*}

\section{Notes on Convergence Rates}
\label{app:convergence-rates}
Suppose a sequence $\{x_k\}_{k\geq 0}$ converges to $x^*$, for $q=1$, define the limit of the ratio of successive errors as 
\begin{equation*}
    p\defeq \lim_{k\rightarrow\infty}\frac{\norm{x_{k+1}-x^*}}{\norm{x_k-x^*}^q}.
\end{equation*}
The rate is referred to as Q-linear when $p\in(0,1)$ and Q-superlinear when $p=0$, where Q stands for quotient. If $q=2$ and $p<\infty$, it is referred to as Q-quadratic convergence. We say the rate is linear-quadratic if
\begin{equation*}
    \norm{x_{k+1}-x^*} \leq p_1\norm{x_k-x^*} + p_2\normsq{x_k-x^*}
\end{equation*}
for $p_1\in(0,1)$ and $p_2<\infty$. Moreover, the R-linear convergence is a weaker notion where R stands for root, and is characterized as the following: there exists a sequence $\{\epsilon_k\}$ such that for all $k\geq0$,
\begin{equation*}
    \norm{x_k-x^*}\leq \epsilon_k
\end{equation*}
where $\{\epsilon_k\}$ converges Q-linearly to $0$. Note that this is a less steady rate as it does not enforce a decrease at every step~\citep{nocedal2006numerical}.
\end{document}